\documentclass[twoside,11pt]{article}

\usepackage{jmlr2e}

\usepackage[margin=1in]{geometry}
\usepackage[utf8]{inputenc} 
\usepackage[T1]{fontenc}   
\usepackage{url}            
\usepackage{booktabs}       
\usepackage{amsfonts}      
\usepackage{nicefrac}      
\usepackage{microtype}      
\usepackage{graphicx}
\usepackage{subfig}
\usepackage{amssymb,amsbsy,amsmath,amsfonts,amssymb,amscd}
\usepackage{comment}
\usepackage{verbatim}
\usepackage[dvipsnames]{xcolor}
\usepackage{enumitem}
\usepackage{algorithm,algpseudocode}
\usepackage{setspace}

\newtheorem{assum}{Assumption}[section]

\newcommand{\EE}{\mathbb{E}}

\newcommand{\Cov}{\mathrm{Cov}}

\newcommand{\rv}{\mathrm{V}}

\newcommand{\wx}{\widetilde{x}}
\newcommand{\wv}{\widetilde{v}}
\newcommand{\ww}{\widetilde{w}}
\newcommand{\wrx}{\widetilde{\mathrm{X}}}
\newcommand{\wrv}{\widetilde{\mathrm{V}}}

\newcommand{\rd}{\,\mathrm{d}}
\newcommand{\revise}[1]{\textcolor{black}{{#1}}}
\newcommand{\blnote}[1]{{{#1}}}

\usepackage{lastpage}
\jmlrheading{22}{2021}{1-\pageref{LastPage}}{10/20; Revised
7/21}{9/21}{20-1205}{Zhiyan Ding and Qin Li}

\ShortHeadings{Langevin Monte Carlo: random coordinate descent and variance reduction}{Ding and Li}
\firstpageno{1}

\begin{document}

\title{Langevin Monte Carlo: random coordinate descent and variance reduction}

\author{\name Zhiyan Ding \email zding49@math.wisc.edu \\
       \addr Mathematics Department\\
       University of Wisconsin-Madison\\
       Madison, WI 53705 USA.
       \AND
       \name Qin Li \email qinli@math.wisc.edu \\
       \addr Mathematics Department and Wisconsin Institutes for Discovery\\
       University of Wisconsin-Madison\\
       Madison, WI 53705 USA.}

\editor{Anthony Lee}

\maketitle

\begin{abstract}
Langevin Monte Carlo (LMC) is a popular Bayesian sampling method. For the log-concave distribution function, the method converges exponentially fast, up to a controllable discretization error. However, the method requires the evaluation of a full gradient in each iteration, and for a problem on $\mathbb{R}^d$, this amounts to $d$ times partial derivative evaluations per iteration. The cost is high when $d\gg1$. In this paper, we investigate how to enhance computational efficiency through the application of RCD (random coordinate descent) on LMC. There are two sides of the theory:

\begin{itemize}
\item By blindly applying RCD to LMC, one surrogates the full gradient by a randomly selected directional derivative per iteration. Although the cost is reduced per iteration, the total number of iteration is increased to achieve a preset error tolerance. Ultimately there is no computational gain;
\item We then incorporate variance reduction techniques, such as SAGA (stochastic average gradient) and SVRG (stochastic variance reduced gradient), into RCD-LMC. It will be proved that the cost is reduced compared with the classical LMC, and in the underdamped case, convergence is achieved with the same number of iterations, while each iteration requires merely one directional derivative. This means we obtain the best possible computational cost in the underdamped-LMC framework.
\end{itemize}
\end{abstract}


\section{Introduction}
Monte Carlo Sampling is one of the core problems in Bayesian statistics, data assimilation~\citep{Reich2011}, and machine learning \citep{MCMCforML}, with wide applications in atmospheric science \citep{FABIAN198117}, petroleum engineering~\citep{PES}, epidemiology~\citep{COVID_travel}, in the form of inverse problems, volume computation \citep{Convexproblem}, and bandit optimization \citep{pmlr-v119-mazumdar20a,ATTS,10.5555/2503308.2343711}.

Let $f(x)$ be $\mu$-strongly convex with its gradient being $L$-Lipschitz in $\mathbb{R}^d$, and define the target density function $p \propto e^{-f}$. Then $p(x)$ is a log-concave probability density function. To sample from the distribution induced by $p(x)$ amounts to finding an $x\in\mathbb{R}^d$ (or a list of samples that can be regarded as i.i.d. (independent and identically distributed) drawn from this distribution.

The literature is very rich on sampling, and there are many approaches~\citep{Neal2001,Neal1993,Robert2004,doi:10.1063/1.1699114,MCSH,Geman1984,DUANE1987216,HMC}. In this paper, we only study variations of Langevin Monte Carlo (LMC) ~\citep{doi:10.1063/1.436415,PARISI1981378,roberts1996,10.5555/3044805.3045080,10.5555/2969442.2969566}. LMC, including the classical LMC and its variations, can be regarded as one subcategory of Markov chain Monte Carlo (MCMC). They are attractive mostly due to the fast convergence rate. Indeed, the classical LMC and its most variations can be regarded as discrete versions of some Langevin dynamics, a set of stochastic differential equations (SDE) that roughly follow the gradient of $f$, with some added Brownian motion terms. Usually, the SDEs are designed so the distribution converges exponentially fast in time to the target distribution~\citep{Markowich99onthe}. LMC, therefore, viewed as the discrete versions (such as the Euler-Maruyama method) of the SDEs, converge also almost exponentially fast, up to a discretization error. The non-asymptotic analysis for the methods is studied in detail in recent years~\citep{dalalyan2018sampling,DALALYAN20195278,durmus2018analysis,Cheng2017UnderdampedLM,dalalyan2018sampling,eberle2018couplings,dwivedi2018logconcave}.

\revise{However, like many other sampling methods, the numerical cost of classical (both overdamped and underdamped) LMC methods increases with respect to $d$, the dimensionality of the problem. While the number of iterations to achieve a preset error tolerance already depends on $d$, the cost per iteration may also increase with respect to $d$. Indeed, there are examples in which, to compute the gradients, $d$ partial derivatives need to be computed separately (see Section~\ref{sec:discuss}), and this adds another $d$ folds of cost per iteration.}

The cost of evaluating the gradient has triggered a large number of studies centering around the ``gradient-free" property. Many approaches were investigated, including Important Sampling~\citep{IM1989,Neal2001,del2006sequential,doi:10.1080/10618600.2015.1062015}, ensemble methods~\citep{Iglesias_2013,EKS,Reich2011}, and random walks~\citep{10.2307/2242610,roberts1996,10.1093/biomet/83.1.95,dwivedi2018logconcave}. These methods shed light on eliminating the evaluation of the gradients, but to this day we have not found a method that is proved to be more competitive with LMC.

\subsection{Contribution}
In this article, we would like to study how to enhance the numerical efficiency of LMC, with a special eye on improving the cost's dependence on $d$. The strategy we take is to develop alternatives for evaluating the gradients. More specifically, we explore how to incorporate random coordinate descent (RCD) in LMC. \revise{RCD is a technique developed for optimization problems, and it randomly chooses one partial derivative, as a replacement of the full gradient in gradient descent (GD).} Since only one partial derivative is calculated in each iteration, the cost is reduced by $d$-folds per iteration. If one carefully controls the iteration number, the overall cost can also be reduced in certain scenarios, as proved in the optimization literature~\citep{doi:10.1137/100802001}.

We study if applying RCD to LMC can bring us similar benefits. The theoretical guarantee we obtain in this paper shows that there are two sides of the story.

Firstly, we blindly apply RCD to LMC. This is to replace the full gradient in LMC with a randomly selected partial derivative in each iteration, exactly as RCD in optimization. This is done to both the overdamped LMC and underdamped LMC, resulting in two algorithms: RCD-O-LMC, and RCD-U-LMC respectively. Since only one directional derivative is calculated instead of the full gradient in each iteration, the cost is reduced by $d$-folds per iteration. However, we can find counterexamples that show this blind application of RCD induces a large error term that represents the high variance produced by the random direction selection procedure. Exactly due to this large error term, more iterations are required. Ultimately there is no saving in the numerical cost.

Secondly, we study variance reduction techniques. Two main techniques are explored: SAGA (stochastic average gradient)~\citep{SAGA-2013,SAGA-2014} and SVRG (stochastic variance reduced gradient)~\citep{Johnson_Zhang}, hoping to reduce the large error term mentioned above, saving the numerical cost in the end. These methods were developed as optimization techniques to improve the convergence rate of SGD (stochastic gradient descent)~ \citep{robbins1951,kiefer1952,10.1007/978-3-7908-2604-3_16}, an algorithm that looks for the minimizer of an objective function that has the form of $f(x) = \sum^N_{n=1}f_n(x)$. Such techniques were then applied to improve stochastic gradient Markov chain Monte Carlo (SG-MCMC)~\citep{10.5555/3104482.3104568,NIPS2016_9b698eb3,NIPS2016_03f54461,Baker2019}. The complexities of these methods are studied in \citep{pmlr-v80-chatterji18a,Zou2018SubsampledSV,pmlr-v80-zou18a,pmlr-v139-zou21b}. A summary of these results can be found in \citep[Section 2.3]{pmlr-v139-zou21b}. We investigate how to incorporate these techniques to enhance the performance of RCD-LMC. In our case $f$ does not necessarily have the form of $\sum^N_{n}f_n(x)$, however, viewing $\nabla f = \sum_k\partial_kf \textbf{e}^k$, variance reduction techniques for SGD can still be borrowed. The methods with variance reduction integrated are called Randomized Coordinates Averaging Descent Overdamped/Underdamped LMC (RCAD-O/U-LMC) and Stochastic Variance Reduced Gradient Overdamped/Underdamped LMC (SVRG-O/U-LMC). We will show that, with either variance reduction technique, in the underdamped setting, the new methods require the same number of iterations as the classical U-LMC~\citep{Cheng2017UnderdampedLM} to achieve a small preset error tolerance $\epsilon$, but the number of directional derivative per iteration is $1$ instead of $d$. This automatically saves $d$ folds of computation.

We summarize all the convergence results in Table~\ref{table:result} (assuming computing the full gradient costs $d$ times of computing one directional derivative and arbitrary initial distribution). We do not present the dependence on the conditioning $\kappa$ in the table, but they are included in the later discussions. The presented result assumes the large $\kappa$ is a secondary concern compared to the large $d$ ($\kappa\ll d$).

\begin{table}[ht]
\begin{center}
\begin{tabular}{ |c|c|c|}
\hline
Algorithm&Number of iterations&Cost\\
\hline
O-LMC&$\widetilde{O}\left(d/\epsilon\right)$&$\widetilde{O}\left(d^2/\epsilon\right)$\\
\hline
U-LMC&$\widetilde{O}\left(d^{1/2}/\epsilon\right)$&$\widetilde{O}\left(d^{3/2}/\epsilon\right)$\\
\hline
RCD-O-LMC &$\widetilde{O}\left(d^{2}/\epsilon^2\right)$&$\widetilde{O}\left(d^{2}/\epsilon^2\right)$\\
\hline
RCD-U-LMC &$\widetilde{O}\left(d^{2}/\epsilon^2\right)$&$\widetilde{O}\left(d^{2}/\epsilon^2\right)$\\
\hline
SVRG-O-LMC &$\widetilde{O}\left(d^{3/2}/\epsilon\right)$&$\widetilde{O}\left(d^{3/2}/\epsilon\right)$\\
\hline
SVRG-U-LMC &$\widetilde{O}\left(\max\{d^{4/3}/\epsilon^{2/3},d^{1/2}/\epsilon\}\right)$&$\widetilde{O}\left(\max\{d^{4/3}/\epsilon^{2/3},d^{1/2}/\epsilon\}\right)$\\
\hline
RCAD-O-LMC &$\widetilde{O}\left(d^{3/2}/\epsilon\right)$&$\widetilde{O}\left(d^{3/2}/\epsilon\right)$\\
\hline
RCAD-U-LMC &$\widetilde{O}\left(\max\{d^{4/3}/\epsilon^{2/3},d^{1/2}/\epsilon\}\right)$&$\widetilde{O}\left(\max\{d^{4/3}/\epsilon^{2/3},d^{1/2}/\epsilon\}\right)$\\
 \hline
\end{tabular}
\caption{Number of iterations and numerical cost to achieve $\epsilon$-accuracy. The results for the classical O-LMC and U-LMC come from~\citep{DALALYAN20195278} and~\citep{Cheng2017UnderdampedLM,dalalyan2018sampling} respectively. The notation $\widetilde{O}(f)$ omits the possible $\log$ terms. For the overdamped cases, we assume the Lipschitz continuity for both the gradient term and the hessian term, and for the underdamped cases, Lipschitz continuity is only assumed for the gradient term.}
\label{table:result}
\end{center}
\end{table}

\subsection{Discussions on assumptions, and relation to the literature}\label{sec:discuss}
\revise{Throughout the paper we assume that computing one partial derivative costs $1/d$ of computing the full gradient. This assumption may not hold in some applications in which the full gradient can be computed efficiently. This happens in logistic regression, support vector machines (SVM), and deep learning where backward propagation is heavily relied on~\citep{backpropagation}.}

\revise{However, there are many examples where this assumption indeed holds true. For example, given a graph with nodes ${\cal N} = \{1,2,\dotsc,d\}$ and directed edges ${\cal E} \subset \{ (i,j) : i,j \in {\cal N} \}$, suppose there is a scalar variable $x_i$ associated with each node $i=1,2,\dotsc,d$, and that the function $f$ has the form 
\[
f(x) = \sum_{(i,j) \in {\cal E}} f_{ij}(x_i,x_j)\,.
\] 
Then the partial derivative of $f$ with respect to $x_i$ is given by
\[
\frac{\partial f}{\partial x_i} = \sum_{j : (i,j) \in {\cal E}} \frac{\partial f_{ij}}{\partial x_i} (x_i,x_j) + \sum_{l : (l,i) \in {\cal E}} \frac{\partial f_{li}}{\partial x_i} (x_l,x_i)\,.
\]
Note that the number of terms in the summations in this expression equals the number of edges in the graph that touch node $i$, the expected value of which is about $2/d$ times the total number of edges in the graph. Meanwhile, evaluation of the full gradient would require evaluation of partial derivatives with respect to both component $i$ and $j$ of the function $f_{ij}$ for {\em all} edges in the graph, leading to a factor-of-$d$ difference in evaluation cost. This setup is encountered in many graph-based problems, such as graph-based label propagation in semi-supervised learning~\citep{Yoshua20}, finding the densest $k$-subgraph~\citep{doi:10.1080/10556788.2019.1595620}, and finding the most likely assignments in continuous pairwise graphical models~\citep{Rue2005}. It is also observed in large-scale SVM problems (when solved in dual form)~\citep{10.1145/1390156.1390208} as well.}

\revise{The assumption also incurs in many PDE-constrained inverse problems as well. The evaluation of the objective function $f$ typically accounts for one PDE solve. As a comparison, each partial derivative amounts to 2 PDE solves (one forward encoding the input data and one adjoint presenting the output test function), and in each iteration, to have the full gradient (also termed Fr\'echet derivative in that context), many partial derivatives need to be computed~\citep{pde_inverse_book}. This leads to a large number of PDE solves for only one iteration.}

Another assumption we make in this paper is that the conditioning of the objective function $f$ is only moderately big, in the sense that its value is not comparable to the largeness of $d$. This means the underperformance of RCD-LMC and the outperformance of the variance reduced RCD-LMC hold uniformly true for all $f$ that have moderate condition numbers. If $f$ has bad conditioning, the Lipschitz constants for each direction are then drastically different. When this happens, one could potentially choose coordinates at different rates to reflect the skewness of $f$. This indeed happens in RCD surrogating GD, where the stiffer directions get chosen more frequently. If we drop the assumption and let $f$ be very skewed as well, a biased selection process could potentially enhance our methods even more. This topic is discussed in two separate contributions, see~\citep{ding2020random} and~\citep{ding2020random2} for overdamped and underdamped cases respectively. For the completeness of the paper, in Remark~\ref{rmk:non_uniform1} we discuss how to sample from a non-uniform coordinate selection process, and in Remark~\ref{rmk:non_uniform2} we argue why it does not bring computational benefits when $f$ is well-conditioned.

There are strong connections between the methods proposed in the paper and the existing literature. We first mention that most of the techniques are borrowed from optimization. RCD is one strategy that surrogates the full gradient by a partial derivative in GD and has been shown to be useful when the objective function $f$ has certain structures. Similar ideas were used in SGD when the objective function has the form of $f=\sum^K_{k=1}f_k$, and in SGD, a $\nabla f_i$ is randomly selected to surrogate $\nabla f$ per iteration. It was observed that although SGD reduces the per iteration cost by $K$, the total number of iteration is increased due to the induced higher variance. For this reason, multiple variance reduction techniques were developed, including SVRG (stochastic variance reduced gradient)~\citep{Johnson_Zhang}, and SAGA (stochastic average gradient)~\citep{SAGA-2013,SAGA-2014}, all adopted in the current paper.

\revise{We also would like to mention that the authors of the paper have already presented part of the results of this paper in~\citep{ding2020variance}. In that paper, it was already observed that blindly applying RCD to LMC will not enhance the efficiency, and some variance reduction is needed. RCAD-O/U-LMC, built upon integrating SAGA to RCD-LMC, was proposed and studied there. That work was our first attempt in addressing similar issues. With the insights gained there, we greatly extended the results, and provide the general recipe in the current paper. The main extension can be summarized as follows:
\begin{itemize}
\item We give the non-asymptotic convergence result of RCD-O/U-LMC in Theorem~\ref{thm:discreteconvergence} and Theorem~\ref{thm:discreteconvergenceULD}. These two bounds were not included in~\citep{ding2020variance}.
\item We propose and rigorously analyze a new variance reduction method SVRG-O/U-LMC. This method, similar to RCAD-O/U-LMC, can also be shown to reduce the computational cost.
\item Moreover, we present a systematic understanding of how the variance presents in the convergence rate, and how to integrate variance reduction techniques to reduce the cost. The overarching results that summarize, in a general form, the impact of involving randomness in the coordinate selection process are presented in Theorem~\ref{imlem:olmc} and Theorem~\ref{imlem:ulmc}. The two theorems address the overdamped and the underdamped cases respectively. The blind application of RCD, and the incorporation of SAGA and SVRG, then can be viewed as three different examples under this framework. We emphasize that though we only study two variance reduction techniques in the current paper, Theorem~\ref{imlem:olmc} and~\ref{imlem:ulmc} are general enough to treat other methods as well. Any new variance reduction method, when cast into the random coordinate LMC framework, can be analyzed in a similar fashion.
\end{itemize}
}

\subsection{Organization}
In Section~\ref{sec:ingredients}, we discuss the essential ingredients of our methods. In Section~\ref{sec:notation}, we unify the notations and assumptions used in our methods. In Section~\ref{sec:mainresult}, we present the convergence results of RCD-LMC and also provide a counter-example to demonstrate the fact that RCD-LMC does not save numerical cost. In Section~\ref{sec:variancereduction}, we introduce the methods that incorporate variance reduction techniques: SVRG-O/U-LMC and RCAD-O/U-LMC, and present the theorems on the convergence and the numerical cost. In Section \ref{sec:Numer}, numerical evidence will be given to demonstrate the improved efficiency of our new methods.  In Section~\ref{proofOLMC}-\ref{proofofULMC}, we provide the proof of the theorems for overdamped, and underdamped settings respectively. We note that in subsection~\ref{sec:olmc_iteration} and subsection~\ref{sec:ulmc_iteration} we provide the overarching results for O/U-LMC when randomness is involved in the coordinate selection process. The later subsections are devoted to the realization of these two theorems on specific algorithms as examples. The discussion on the best possible numerical cost is heavily technical and thus delayed to these two subsections (Remark \ref{re:6.1} and Remark \ref{re:7.1}). Some technical lemmas used in these two sections are postponed to the appendix.

\section{Essential ingredients}\label{sec:ingredients}
In this section, we discuss the essential ingredients of our methods: the overdamped and underdamped
Langevin dynamics and the associated Monte Carlo methods (O-LMC and U-LMC); random
coordinate descent (RCD); and variance reduction techniques, including both SVRG and SAGA.

\subsection{Overdamped Langevin dynamics and O-LMC}\label{sec:o-lmc}
The O-LMC method comes from the overdamped Langevin dynamics, a stochastic differential equation that writes:
\begin{equation}\label{eqn:Langevin}
\rd X_t=-\nabla f(X_t)\rd t+\sqrt{2}\rd B_t\,.
\end{equation}
This SDE characterizes the trajectory of $X_t$. Two forcing terms $\nabla f(X_t)\rd t$ and $\rd B_t$ compete: the former drives $X_t$ to the minimum of $f$ and the latter provides Brownian motion and thus small oscillations along the trajectory. The initial data $X_0$ is a random variable drawn from a given distribution induced by $q_0(x)$. Denote $q(x,t)$ the probability density function of $X_t$, it is a well-known result that $q(x,t)$ satisfies the following Fokker-Planck equation:
\begin{equation}\label{eqn:FKPKLangevin}
\partial_t q=\nabla\cdot(\nabla fq+\nabla q)\,,\quad\text{with}\quad q(x,0) = q_0\,,
\end{equation}
and furthermore, $q(x,t)$ converges to the target density function $p(x) \propto e^{-f}$ exponentially fast in time~\citep{Markowich99onthe}.

The overdamped Langevin Monte Carlo (O-LMC), as a sampling method, is simply a discrete-in-time version of the SDE~\eqref{eqn:Langevin}. A standard Euler-Maruyama method applied on the equation gives:
\begin{equation}\label{eqn:update_ujn}
x^{m+1}=x^m-\nabla f(x^m)h+\sqrt{2h}\xi^{m}\,,
\end{equation}
where $h$ is the small time-step and $\xi^{m}$ is i.i.d. drawn from $\mathcal{N}(0,I_d)$ with $I_d$ being the identity matrix of size $d$. Since~\eqref{eqn:update_ujn} approximates~\eqref{eqn:Langevin}, the density of $x^m$ becomes $p(x)$ as $m\to\infty$, up to a discretization error. It was proved in~\citep{DALALYAN20195278} that the convergence to $\epsilon$ is achieved within $\widetilde{O}(d/\epsilon)$ iterations if the hessian of $f$ is Lipschitz. If this hessian is not Lipschitz, the number of iterations is shown to be $\widetilde{O}(d/\epsilon^2)$. In many real applications, the gradient of $f$ is not explicitly known and some approximation is used, introducing another layer of numerical error. In~\citep{DALALYAN20195278}, the authors discussed the effect of such error, under the assumption that the error term has a bounded variance.

\subsection{Underdamped Langevin dynamics and U-LMC}\label{sec:u-lmc}
The underdamped Langevin dynamics \citep{10.5555/3044805.3045080} is characterized by the following SDE system:
\begin{equation}\label{eqn:underdampedLangevin}
\left\{\begin{aligned}
&\rd X_t = V_t\rd t\\
&\rd V_t = -2 V_t\rd t-\gamma\nabla f(X_t)\rd t+\sqrt{4 \gamma}\rd B_t
\end{aligned}\right.\,,
\end{equation}
where $\gamma>0$ is a parameter to be tuned. Denote $q(x,v,t)$ the probability density function of $(X_t,V_t)$, then $q$ satisfies the Fokker-Planck equation
\[
\partial_tq=\nabla\cdot \left(\begin{bmatrix}
-v\\
2v+\gamma\nabla f
\end{bmatrix}q+\begin{bmatrix}
0 & 0\\
0 & 2\gamma
\end{bmatrix}\nabla q\right)\,,
\]
and under mild conditions, in $t\to\infty$, it converges to $p_2(x,v)\propto\exp(-(f(x)+|v|^2/2\gamma))$, making the marginal density function for $x$ the target $p(x)$.

The underdamped Langevin Monte Carlo algorithm, U-LMC, can be viewed as a numerical solver to \eqref{eqn:underdampedLangevin}. In each step, we sample
$(x^{m+1},v^{m+1})\in\mathbb{R}^{2d}$ as a Gaussian random variable determined by $(x^m,v^m)$ with the following expectation and covariance:
\begin{equation}\label{distributionofZ}
\begin{aligned}
&\EE x^{m+1}=x^m+\frac{1}{2}\left(1-e^{-2h}\right)v^m-\frac{\gamma}{2}\left(h-\frac{1}{2}\left(1-e^{-2h}\right)\right)\nabla f(x^m)\,,\\
&\EE v^{m+1}=v^me^{-2h}-\frac{\gamma}{2}\left(1-e^{-2h}\right)\nabla f(x^m)\,,\\
&\Cov\left(x^{m+1}\right)=\gamma\left[h-\frac{3}{4}-\frac{1}{4}e^{-4h}+e^{-2h}\right]\cdot I_d\,,\ \Cov\left(v^{m+1}\right)=\gamma\left[1-e^{-4h}\right]\cdot I_d\,,\\
&\Cov\left(x^{m+1}\,,v^{m+1}\right)=\frac{\gamma}{2}\left[1+e^{-4h}-2e^{-2h}\right]\cdot I_d\,.
\end{aligned}
\end{equation}

We here used the notation $\EE$ to denote the expectation, and $\Cov(a,b)$ to denote the covariance of $a$ and $b$. If $b=a$, we abbreviate it to $\Cov(a)$. The scheme can be interpreted as sampling from the following dynamics in each time interval:
\begin{equation*}
\left\{\begin{aligned}
&\mathrm{X}_t=x^m+\int^t_0 \mathrm{V}_s\rd s\\
&\mathrm{V}_t=v^me^{-2t}-\frac{\gamma}{2}(1-e^{-2t})\nabla f(x^m)+\sqrt{4 \gamma}e^{-2 t}\int^t_0e^{2 s}\rd B_s
\end{aligned}\right.
\end{equation*}
with $x^{m+1}=\mathrm{X}_{h}$ and $v^{m+1}=\mathrm{V}_{h}$, and $h$ is the time step.

The underdamped Langevin Monte Carlo demonstrates a faster convergence rate~\citep{Cheng2017UnderdampedLM} than O-LMC. {Due to the introduction of $V$ in~\eqref{eqn:underdampedLangevin}, the trajectory
$X_t$ is smoother than that of $V_t$ whose smoothness is determined by the Brownian motion $B_t$. As a result, a higher-order discretization is possible for this augmented dynamics.} Without the assumption on the hessian of $f$ being Lipschitz, the number of iterations is $\widetilde{O}(\sqrt{d}/\epsilon)$ to achieve $\epsilon$ accuracy.  A high (3rd) order discretization was discussed for~\eqref{eqn:underdampedLangevin} in~\citep{mou2019highorder}, further enhancing the numerical accuracy to $\widetilde{O}(d^{1/4}/\epsilon^{1/2})$ when $f$ is smooth enough. Similar to the discussion for O-LMC in~\citep{DALALYAN20195278}, the authors in~\citep{dalalyan2018sampling} also studied the error in estimating $\nabla f(x^m)$, but they also assumed the variance is bounded.

\subsection{Random coordinate descent (RCD)}
The idea of RCD is to surrogate the full gradient in Gradient Descent by a randomly selected partial derivative in each iteration~\citep{doi:10.1137/100802001}. Per iteration, only one partial derivative is computed instead of $d$, so there is some hope that total cost can be reduced. More specifically one approximates
\begin{equation}\label{eqn:randomfinitedifferenceRD}
\nabla f\approx d\left(\nabla f(x)\cdot \textbf{e}^{r}\right)\textbf{e}^{r}\,,
\end{equation}
where $\textbf{e}^i$ is the $i$-th unit direction  and $r$ is randomly drawn from $1,2,\cdots,d$. This approximation is consistent in the expectation sense because
\[
\EE_{r}\left(d\left(\nabla f(x)\cdot \textbf{e}^{r}\right)\textbf{e}^{r}\right)=\nabla f(x)\,.
\]
It was shown in~\citep{Ste-2015,RB2011,doi:10.1137/100802001} that cost reduction is indeed observed, especially when the objective function $f$ is highly skewed in a high dimensional space. Similar ideas were discussed in~\citep{SPSAanalyse1,SPSAanalyse2} where SPSA, simultaneous perturbation stochastic approximation, was proposed.

There are many different ways to compute directional derivatives. Automatic differentiation is a strategy used often when the underlying function is composed of many simpler operations, and such composition is explicitly known~\citep{JMLR:v18:17-468}. When the function is complicated itself, one could use the most basic finite differencing method. That is to use $\partial_i f(x)\approx \frac{f(x+\eta \textbf{e}^i)-f(x-\eta \textbf{e}^i)}{2\eta}$ at a sacrifice of $O(\eta^2)$ numerical error (assuming enough smoothness of $f$). For special problems such as the PDE-based inverse problems, raising in atmospheric science and epidemiology~\citep{COVID_travel}, one could further translate the derivative computation into a combination of one forward and one adjoint PDE solves~\citep{OGinverse}. Whichever strategy one employs, in some applications, it is true that when $\nabla f$ is not explicitly known, the computation of the full gradient costs $d$ times of computing one directional derivative. This makes the saving on $d$ extremely important.

\subsection{Variance reduction techniques: SVRG and SAGA}\label{sec:SVRG+SAGA}
SVRG (stochastic variance reduced gradient)\citep{Johnson_Zhang} and SAGA~\citep{SAGA-2014} (a modified version of SAG~\citep{SAGA-2013}, stochastic average gradient) are optimization techniques that are widely used in reducing variance in randomized solvers. In particular, they are introduced to enhance the numerical performance of SGD (stochastic gradient descent). SGD is a stochastic version of Gradient Descent (GD) that looks for the minimization of an objective function that has the form of $f = \sum_{n=1}^N f_n$ with $N\gg 1$. GD requires $\nabla f = \sum_{n=1} \nabla f_n$ evaluated at the current sample in each iteration, which amounts to a computation of $N$ gradients. In SGD, one merely randomly selects one $\nabla f_r$ to represent the full $\nabla f$. Per iteration, this reduces the cost by $N$ folds. However, due to the random selection process, high variance is induced and that brings large errors. More iterations are then needed to achieve a preset accuracy. SVRG and SAGA are algorithms proposed to reduce this variance, hoping to eliminate the error induced by the randomness and save computational cost overall. We now describe the procedure of SVRG and SAGA.

In SVRG, one has a preset iteration number $\tau$, and the full gradient $\nabla f$ is computed only once every $\tau$ steps. Between the epochs, per step, only one $f_r$ is selected at random to represent the new gradient. In particular, call $\widetilde{x}$ the sample obtained at $k\tau$-th step with $k$ being an integer, we have the full gradient at this point $\nabla f(\widetilde{x})$. In the following $\tau-1$ steps, per iteration, one $f_r$ is chosen at random and the new gradient is approximated by
\begin{equation}\label{eqn:SVRGapproximation}
\nabla f(x)\approx \nabla f(\widetilde{x})+N\left(\nabla f_r(x)-\nabla f_r(\widetilde{x})\right)\,.
\end{equation}
After $\tau$ steps, $\widetilde{x}$ is updated and one evaluates the full $\nabla f(\widetilde{x})$ again. This approximation~\eqref{eqn:SVRGapproximation} is unbiased if $r$ is chosen uniformly in the sense that
\[
\nabla f(x) = \sum_n\nabla f_n(x) = \EE_r\left[\nabla f(\widetilde{x})+N\left(\nabla f_r(x)-\nabla f_r(\widetilde{x})\right)\right]\,.
\]

SAGA is slightly different: it only requires the computation of the full gradient at the initial step. In the later iterations, per iteration, only one $f_r$ is chosen at random to update the full gradient while others are kept untouched. Term $\{g^n_m\}^N_{i=1}$ the $m$-th step approximation to $\nabla f_n$, then $g^n_0 =\nabla f_n(x^0)$. For the following iterations with $m>1$, $r$ is uniformly randomly picked and $g^r_m = \nabla f_r(x^m)$ while others are kept untouched $g^k_m = g^{k}_{m-1}$. We then approximate the full gradient using:
\begin{equation}\label{eqn:SAGAapproximation}
\nabla f(x^m)\approx \sum^N_{n=1}g^n_{m-1}+N(g^r_m-g^r_{m-1})\,.
\end{equation}
Similar to SVRG, the approximation~\eqref{eqn:SAGAapproximation} is unbiased in the sense that
\[
\nabla f(x) = \EE_r\left[\sum^N_{n=1} g^n_{m-1} + N\left(\nabla f_r(x)-g^r_{m-1}\right)\right]\,.
\]

Both techniques are proven to reduce variance for SGD~\citep{Johnson_Zhang,SAGA-2014}.

\section{Notations and classical results}\label{sec:notation}
We unify notions and assumptions, and briefly summarize the classical results in this section.
\subsection{Assumptions and the Wasserstein distance}
We make standard assumptions on $f(x)$:
\begin{assum}\label{assum:Cov}
The function $f$ is second-order differentiable, $\mu$-strongly convex and has an $L$-Lipschitz gradient:
\begin{itemize}
\item[--] Convex, meaning for any $x,x'\in\mathbb{R}^d$:
\begin{equation}\label{Convexity}
f(x)-f(x')-\nabla f(x')^\top (x-x')\geq (\mu/2)|x-x'|^2\,.
\end{equation}
\item[--] Gradient is Lipschitz, meaning for any $x,x'\in\mathbb{R}^d$:
\begin{equation}\label{GradientLip}
|\nabla f(x)-\nabla f(x')|\leq L|x-x'|\,.
\end{equation}
\end{itemize}
\end{assum}
These assumptions together mean $\mu{I}_d\preceq H(f)\preceq L{I}_d$ where $H(f)$ is the hessian of $f$. We also define condition number of $f(x)$ as
\begin{equation}\label{eqn:R}
\kappa=L/\mu\geq 1\,.
\end{equation}

Furthermore, for O-LMC we assume Lipschitz condition of the hessian:
\begin{assum}\label{assum:Hessian}
Hessian of $f$ is H-Lipschitz, meaning for any $x,x'\in\mathbb{R}^d$:
\begin{equation}\label{HessisnLip}
\|H(f)(x)-H(f)(x')\|_2\leq H|x-x'|\,.
\end{equation}
\end{assum}

Throughout the analysis, we use the Wasserstein distance as the quantity to measure the distance between two probability measures. For $p\geq 1$, define the Wasserstein distance to be:
\[
W_p(\mu,\nu) = \left(\inf_{(X,Y)\in \Gamma(\mu,\nu)} \mathbb{E}|X -Y|^p\right)^{1/p}\,,
\]
where $\Gamma(\mu,\nu)$ is the set of distribution of $(X,Y)\in\mathbb{R}^{2d}$ whose marginal distributions are $\mu$ and $\nu$ respectively for $X$ and $Y$. These distributions are called the couplings of $\mu$ and $\nu$. Here $\mu$ and $\nu$ can be either probability measures themselves or the measures induced by probability density functions $\mu$ and $\nu$. In this paper, we only use $p=2$.

\subsection{Classical O/U-LMC}\label{sec:coulmc}
We now review the classical results regarding O-LMC and U-LMC. The two algorithm are summarized in the following:
\begin{algorithm}[htb]
\caption{\textbf{Overdamped/Underdamped Langevin Monte Carlo (O/U-LMC)}}\label{alg:OULMC}
\begin{algorithmic}
\State \textbf{Preparation:}
\State 1. Input: $h$ (time stepsize); $\gamma$ (parameter); $d$ (dimension); $M$ (stopping index); $\nabla f(x)$.
\State 2. Initial: \emph{(overdamped)}: $x^0$ i.i.d. sampled from an initial distribution induced by $q_0(x)$.

\emph{(underdamped)}: $(x^0,v^0)$ i.i.d. sampled from the initial distribution induced by $q_0(x,v)$.

\State \textbf{Run: }\textbf{For} $m=0\,,1\,,\cdots\,,M$

\emph{(overdamped)}: Draw $\xi^{m}$ from $\mathcal{N}(0,I_d)$:
\begin{equation}
x^{m+1}=x^m-\nabla f(x^m) h+\sqrt{2h}\xi^{m}\,.
\end{equation}
\emph{(underdamped)}: Sample $(x^{m+1},v^{m+1})$ as Gaussian random variables with expectation and covariance defined in~\eqref{distributionofZ}.
\State \textbf{end}
\State \textbf{Output:} $\{x^m\}$.
\end{algorithmic}
\end{algorithm}

The non-asymptotic convergence results have been thoroughly discussed in~\citep{DALALYAN20195278,durmus2018analysis,durmus2018highdimensional} and \citep{Cheng2017UnderdampedLM,dalalyan2018sampling} for O-LMC and U-LMC respectively. We merely cite the theorems.

\begin{theorem}\label{thm:convergenceolmc}[\citep{DALALYAN20195278} Theorem 5]
Assume $h<\frac{2}{\mu +L}$ and $f$ satisfies Assumptions \ref{assum:Cov}-\ref{assum:Hessian}. Denote $q^O_m(x)$ the probability density function of $x^m$ computed using O-LMC, and define $W_m=W_2(q^O_{m},p)$, the $L_2$-Wasserstein distance between $q^O_m(x)$ and $p$, then we have:
\begin{equation}\label{eqn:convergenceolmc}
W_m\leq \exp\left(-\mu mh\right)W_0+\left[\frac{Hhd}{2\mu }+3\kappa^{3/2}\mu ^{1/2}hd^{1/2}\right]\,.
\end{equation}
\end{theorem}
Note that there are two parts in the iteration formula for $W_m$, the exponentially decaying term, and the remainder term. Without considering the conditioning constants, the remainder term is of the order of $hd$. To have the Wasserstein distance to be smaller than a preset tolerance, for example, $W_m\leq \epsilon$, we need
\[
\exp\left(-\mu mh\right)\lesssim\epsilon\,,\quad \text{and}\quad hd\lesssim\epsilon\,,
\]
which explains $m=\widetilde{O}(d/\epsilon)$ in Table \ref{table:result}. Since each iteration requires the computation of the full gradient, meaning $d$ partial derivatives, the total cost is then $\widetilde{O}(d^2/\epsilon)$. In~\cite{durmus2018analysis}, the authors improved the dependence of the convergence rate on the condition number, with the Lipschitz continuity of the hessian term relaxed.

For the underdamped LMC, we have:
\begin{theorem}\label{thm:convergenceulmc}[\citep{dalalyan2018sampling} Theorem 2]
Assume $f$ satisfies Assumption~\ref{assum:Cov}, $\gamma=\frac{1}{L}$ and $h\leq \frac{1}{8\kappa^2\mu}$. Denote $q^U_m(x,v)$ the  probability density function of $(x^m,v^m)$ computed using U-LMC, and define $W_m = W_2(q^U_m,p)$, then we have:
\begin{equation}\label{eqn:convergenceulmc}
W_m\leq \sqrt{2}\exp(-0.375mh/\kappa)W_0+h(\kappa d)^{1/2}\,.
\end{equation}
\end{theorem}

The iteration formula is once again composed of two terms: the exponential term and the remainder, with the remainder being of order $hd^{1/2}$. Setting both terms smaller than $\epsilon$, we have
\[
h=\widetilde{O}(d^{-1/2}\epsilon)\,,\quad\text{and}\quad m=\widetilde{O}(d^{1/2}/\epsilon)\,,
\]
making the total number of  partial derivative calculations being $\widetilde{O}(d^{3/2}/\epsilon)$, as shown in Table \ref{table:result}.

\section{Algorithms and Results of RCD-LMC}\label{sec:mainresult}
As seen in Algorithm~\ref{alg:OULMC}, per iteration, $\nabla f$ needs to be evaluated at the current sample, and that prompts $d$ {partial derivative calculations}. If $d$ is big, the computation is expensive.

RCD (random coordinate descent) is a method introduced in optimization to reduce the number of partial derivative calculations from the classical gradient descent. It essentially surrogates the full gradient in the gradient descent method by one directional derivative in each iteration, and thus it naturally reduces the computational cost by $d$ folds per iteration. If we blindly apply this approach to O/U-LMC, we arrive at RCD-O/U-LMC. These two methods were briefly discussed in~\citep{ding2020variance}. It was also shown there that the methods, despite easy to implement and reduce the per-iteration-cost by $d$ folds, the overall cost is, however, not saved due to the larger number of required iterations. This means blindly applying RCD to O/U-LMC is a bad strategy. For the completeness of the paper, we still summarize the results and cite the convergence rate in the following subsections.

We note that the increased numerical cost could be explained by the high variance induced by the randomness in the coordinate selection process. The same issue was encountered in optimization as well, see discussions in~\citep{RB2011,Ste-2015}. As a consequence, a lot of efforts have been placed on tuning the probability of the coordinate-drawing, such as incorporating the information hidden in the directional Lipschitz constants. The application of these strategies should be explored in sampling as well, but it is beyond the scope of the current paper.  We should also emphasize that the counterexample is a worst-case study. There are plenty of examples where the application of RCD to sampling already gives a numerical boost, but in this paper, we would like to be as general as possible.

\subsection{Algorithm}
We apply RCD to both O-LMC and U-LMC. This amounts to replacing the gradient terms in \eqref{eqn:update_ujn} and~\eqref{distributionofZ} using the approximation~\eqref{eqn:randomfinitedifferenceRD} . The new methods are termed RCD-O-LMC and RCD-U-LMC respectively, and we present them in Algorithm \ref{alg:SOU-LMC}.
\begin{algorithm}[htb]
\caption{\textbf{RCD-O/U-LMC}}\label{alg:SOU-LMC}
\begin{algorithmic}
\State \textbf{Preparation:}
\State 1. Input: $h$ (time stepsize); $\gamma$ (parameter); $d$ (dimension); $M$ (stopping index); $f(x)$.
\State 2. Initial: \emph{(overdamped)}: $x^0$ i.i.d. sampled from an initial distribution induced by $q_0(x)$.

\emph{(underdamped)}: $(x^0,v^0)$ i.i.d. sampled from the initial distribution induced by $q_0(x,v)$.

\State \textbf{Run: }\textbf{For} $m=0\,,1\,,\cdots\,,M$

1. Prepare directional derivative: draw $r$ uniformly from $1\,,\cdots\,,d$, and compute:
\begin{equation}\label{eqn:RCD}
F^m=d\partial_rf(x^m)\textbf{e}^{r}\,.
\end{equation}

2. \emph{(overdamped)}: Draw $\xi^{m}$ from $\mathcal{N}(0,I_d)$:
\begin{equation}\label{eqn:update_ujnSD}
x^{m+1}=x^m-F^m h+\sqrt{2h}\xi^{m}\,.
\end{equation}
\emph{(underdamped)}: Sample $(x^{m+1},v^{m+1})$ as Gaussian random variables with expectation and covariance defined in~\eqref{distributionofZ}, replacing $\nabla f(x^m)$ by $F^m$.
\State \textbf{end}
\State \textbf{Output:} $\{x^m\}$.
\end{algorithmic}
\end{algorithm}

\subsection{Convergence results for RCD-O/U-LMC}
We discuss the convergence of Algorithm \ref{alg:SOU-LMC} in this section and compare the results with the classical results for O-LMC~\citep{DALALYAN20195278} and U-LMC~\citep{Cheng2017UnderdampedLM}. 

\subsubsection{Convergence for RCD-O-LMC}
The theorem we show is:
\begin{theorem}\label{thm:discreteconvergence}
Suppose $f$ satisfies Assumptions \ref{assum:Cov}-\ref{assum:Hessian}, and $h$ satisfies
\begin{equation}\label{eqn:RCD-O-LMC_h}
h<\min\left\{\frac{1}{9\kappa^2\mu d},\frac{2}{H^2/(\kappa \mu ^2)+\kappa^2 \mu /d}\right\}\,,
\end{equation}
then we have:
\begin{equation}\label{eqn:thmW2bound}
W_m\leq \exp\left(-\mu mh/4\right)W_0+6d(\kappa h)^{1/2}\,,
\end{equation}
where $W_m=W_2(q^O_{m},p)$ and $q^O_m(x)$ is the probability density function of $x^m$ computed by RCD-O-LMC.
\end{theorem}

We discuss the proof in Section \ref{proofofthm:discreteconvergence}. The statement serves as the guidance to tune parameters and estimates the computational complexity. According to \eqref{eqn:thmW2bound}, for $\epsilon$ accuracy, it suffices to choose
\[
6\kappa dh^{1/2}\leq \frac{\epsilon}{2},\quad\text{and}\quad\exp(-\mu hm/4)\leq \frac{\epsilon}{2W_0}\,,
\]
which yields
\[
h<\frac{\epsilon^2}{24\kappa^2d^2}\,\quad \text{and}\quad mh\geq 4/\mu \log\left(2W_0/\epsilon\right)\,.
\]
For small enough $\epsilon$, this new $h$ restriction is stronger than~\eqref{eqn:RCD-O-LMC_h}, and this implies $O(d^2/\epsilon^2\log(W^O_0/\epsilon))$ iterations are needed. Since each iteration requires only one partial derivative computation, the cost is also $\mathcal{\widetilde{O}}(d^2/\epsilon^2)$. Compared with the classical O-LMC presented in Section~\ref{sec:coulmc}, this cost is in fact $1/\epsilon$ higher.

\revise{\begin{remark}\label{rmk:discussionOLMC}
While the exponential decay term naturally appears due to the convexity, as was done in the classical analysis for O-LMC, the second term of~\eqref{eqn:thmW2bound} is at the order of $O(h^{1/2}d)$, instead of $O(hd)$ as in O-LMC. This deterioration comes from the extra error term induced by replacing $\nabla f(x^m)$ by $F^m$. As a consequence, the core of the proof lies in bounding
\begin{equation}\label{eqn:E}
E^m=\nabla f(x^m)-F^m\,.
\end{equation}
As shown in Lemma \ref{lem:E}, the variance of this term is:
\[
\mathbb{E}(|E^m|^2)\lesssim O(d^2)\,.
\]
When inserted in the updating formula, this contributes a term of order $O(h^2d^2)$ and it dominates all other error sources ($O(h^3d^2)$ as seen in the classical O-LMC). This eventually leads to a worse estimate. See more detailed discussion in Remark \ref{rmk:Emdominant1}.
\end{remark}}

\begin{remark}\label{rmk:non_uniform1}
It is a natural question to ask if uniform sampling is necessary. Indeed, in the optimization literature, RCD outperforms GD when the objective function $f$ is skewed, and the coordinates are chosen in a non-uniform way to reflect such skewness. For sampling, we can also use a non-uniform sampling strategy as well.

To this goal, we first denote the probability of choosing $i$-th direction $\phi_i$ and 
\[
\Phi := \{\phi_1,\phi_2,\dotsc, \phi_d\}\,,\quad\text{with}\quad \phi_i > 0\;(\forall i)\,,\quad\text{and}\quad \sum_{i=1}^d \phi_i=1\,.
\]
Then, in the $m$-th iteration of RCD-LMC, we replace~\eqref{eqn:RCD} by
\begin{equation}\label{eqn:randomfinitedifferencenonuniformRD}
F^m=\frac{1}{\phi_r}\left(\nabla f(x^m)\cdot \textbf{e}^{r}\right)\textbf{e}^{r}\,,
\end{equation}
where $x^m$ is the $m$-th iteration sample. This is a valid calculation because this definition is still consistent with $\nabla f$ in the expectation sense: 
\[
\EE_r\left(\frac{1}{\phi_r}\left(\nabla f(x)\cdot \textbf{e}^{r}\right)\textbf{e}^{r}\right)=\nabla f(x)\,.
\]

However, using a non-uniform sampling strategy will not provide better convergence than the uniform sampling under our setting. This will be discussed in Remark~\ref{rmk:non_uniform2}.
\end{remark}

\subsubsection{Convergence for RCD-U-LMC}
In the underdamped setting, we have the following theorem:
\begin{theorem}\label{thm:discreteconvergenceULD}
Suppose $f$ satisfies Assumption \ref{assum:Cov}. Set $\gamma= 1/L$. If $h$ satisfies
\begin{equation}\label{eqn:RCD-U-LMC_h}
h<\frac{1}{880d\kappa}\,,
\end{equation}
then, we have:
\begin{equation}\label{eqn:thmW2boundULD}
W_m\leq 4\exp\left(-hm/(8\kappa)\right)W_0+Cdh^{1/2}\,,
\end{equation}
where $C=100(\kappa/\mu )^{1/2}$, $W_m=W_2(q^U_{m},p_2)$, and $q^U_m(x,v)$ is the probability density function of $(x^m,v^m)$ computed by RCD-U-LMC.\end{theorem}
We discuss the proof in Section \ref{sec:proofofthmdiscreteconvergenceULD}. To obtain $\epsilon$ accuracy, we simply set both terms $<\epsilon/2$ in~\eqref{eqn:thmW2boundULD}, and it yields, in addition to~\eqref{eqn:RCD-U-LMC_h}:
\[
h<\frac{\epsilon^2\mu}{10^4\kappa d^2}\,,\quad\text{and}\quad  mh\geq 8\kappa\log\left(8W_0/\epsilon\right)\,.
\]
This leads to $O\left(d^2/\epsilon^2\log\left(W_0/\epsilon\right)\right)$ number of iterations, and thus the cost is $\mathcal{\widetilde{O}}(d^2/\epsilon^2)$. Compared with the classical U-LMC as discussed in Section~\ref{sec:coulmc}, this cost is $d^{1/2}/\epsilon$ larger. 

\revise{Similar to the overdamped case, the estimate for RCD-U-LMC is worse than that for U-LMC due to the newly induced error $E^m=\nabla f(x^m)-F^m$. Once again this term $\EE|E^m|^2$ provides a larger error source than other terms ($O(h^3d)$ as seen in the classical U-LMC), and is, upon the application of the iteration formula, represented in the second term in \eqref{eqn:thmW2boundULD}. See more details in Remark \ref{rmk:Emdominant2}.}

\subsection{A counter-example}\label{sec:failuerofrandomsearch}
The two theoretical results above suggest that no improvement is seen when RCD is blindly applied to O/U-LMC, if not worse. We cannot argue the sharpness of the two results, but we do provide a counter-example to demonstrate that including RCD to O/U-LMC may not bring computational advantage.

The example is given to show RCD-U-LMC. We assume
\begin{equation}\label{eqn:counterexample}
    q_0(x,v)=\frac{1}{(4\pi)^{d/2}}\exp(-|x-\textbf{u}|^2/2-|v|^2/2)\,,\quad p_2(x,v)=\frac{1}{(2\pi)^{d/2}}\exp(-|x|^2/2-|v|^2/2)\,,
\end{equation}
where $\textbf{u}\in\mathbb{R}^d$ satisfies $\textbf{u}_i=1/8$ for all $i$. Denote $\{(x^m,v^m)\}$ the sample computed through Algorithm \ref{alg:SOU-LMC} (underdamped) with stepsize $h$. Denote $q_m$ the probability density function of $(x^m,v^m)$, then we can show $W_2(q_m,p_2)$ cannot converge too fast.
\begin{theorem}\label{thm:badexampleW22}[Theorem 4.1 in \cite{ding2020variance}]
For the example above, choose $\gamma=1$ in Algorithm~\ref{alg:SOU-LMC}, and let $d,h$ satisfy
\[
d>1872,\quad h<\frac{1}{1440^2d}\,,
\]
then
\begin{equation}\label{eqn:badexampleW2bound2}
W_m\geq \exp\left(-2mh\right)\frac{\sqrt{d}}{1024}+\frac{d^{3/2}h}{2304}\,,
\end{equation}
where $W_m=W_2(q^U_{m},p_2)$, and $q^U_m(x,v)$ is the probability density function of $(x^m,v^m)$ computed by RCD-U-LMC.
\end{theorem}

We note the second term in \eqref{eqn:badexampleW2bound2} is rather big. The smallness comes from $h$, the stepsize, and it needs to be small enough to balance out the influence from $d^{3/2}\gg 1$. This puts a strong restriction on $h$. Indeed, to have $\epsilon$-accuracy, $W(q^U_m,p_2)\leq\epsilon$, we need both terms smaller than $\epsilon$, and this term suggests that $h\leq\frac{2304\epsilon}{d^{3/2}}$ at least. And when combined with restriction from the first term, we arrive at the conclusion that at least $\widetilde{O}(d^{3/2}/\epsilon)$ iterations are needed, and thus $\widetilde{O}(d^{3/2}/\epsilon)$ {partial derivative calculations} are required. The $d$ dependence is $d^{3/2}$. This is exactly the same as what U-LMC requires, meaning RCD-U-LMC brings no computational advantage over U-LMC.

\begin{remark}
All discussions on convergence for U-LMC methods are conducted without assuming the hessian of $f$ is Lipschitz. A natural question is: If $f$ has a higher regularity, is the method converging faster? The example above gives a negative answer. If we apply the classical U-LMC to this example, the number of iterations to achieve $\epsilon$-accuracy in $W_2$-distance is at least $\widetilde{O}(d^{1/2}/\epsilon)$. This coincides with the results obtained when only $\nabla f$ is assumed to be Lipschitz. For this reason, throughout the paper we do not require high regularity of $f$ under the U-LMC framework.
\end{remark}
\section{Variance reduction techniques}\label{sec:variancereduction}
Blindly applying RCD to LMC does not bring any numerical savings. In the analysis, the biggest contribution to the large error (the remainder term) is induced by the process of randomly selecting directional derivatives. In this section we propose two variance reduction methods inspired by SVRG~\citep{Johnson_Zhang} and SAGA~\citep{SAGA-2014} and apply them to RCD-O/U-LMC (and thus four algorithms in total).  We will prove that while the numerical cost per iteration is reduced by $d$-folds, the number of iterations needed for achieving the preset accuracy is mostly unchanged compared to the classical LMC. This ultimately saves the total cost.

\subsection{Algorithms}
\subsubsection{SVRG based variance reduction method}
As presented in Section~\ref{sec:SVRG+SAGA}, SVRG is an optimization technique introduced to reduce variance in SGD in its random selection process. Only one representative gradient is computed per iteration, unless the iteration number is an integer times $\tau$, when all gradients are evaluated. It has been proved in~\citep{Johnson_Zhang} that the approach enhances the SGD performance.

Inspired by this method, we propose SVRG-O/U-LMC, as summarized in Algorithm~\ref{alg:SVRG-OU-LMC}. In this algorithm, we compute the full gradient $\nabla f$ at $k\tau$ time-step for all integers $k$, and during the epochs, per iteration, we merely update one directional derivative chosen at random.

\begin{algorithm}[htb]
\caption{\textbf{SVRG-O(U)-LMC}}\label{alg:SVRG-OU-LMC}
\begin{algorithmic}
\State \textbf{Preparation:}
\State 1. Input: $h$ (time stepsize); $\tau$ (epoch length); $\gamma$ (parameter); $M$ (stopping index); $d$ (dimension) and $f(x)$.
\State 2. Initial: \emph{(overdamped)}: $x^0$ i.i.d. sampled from an initial distribution induced by $q_0(x)$.

\emph{(underdamped)}: $(x^0,v^0)$ i.i.d. sampled from the initial distribution induced by $q_0(x,v)$.
\State \textbf{Run: }\textbf{For} $m=0\,,1\,,\cdots\,,M$

\textbf{if} $m$ mod $\tau=0$ \textbf{then} update $\widehat{g}$ and compute flux $F^m$:
\begin{equation}\label{eqn:g1SVRGupdate2}
F^m_i = \widehat{g}_i=\partial_if(x^0),\quad \ 1\leq i\leq d\,.
\end{equation}

\textbf{otherwise}

Draw a random number $r^m$ uniformly from $1,2,\cdots,d$ and compute
\begin{equation}\label{randomgradientapproximationSVRG}
g^m_{r^m} = \partial_{r^m}f(x^m)\,,\quad F^m=\widehat{g}+d\left(g^m_{r^m}-\widehat{g}_{r^m}\right)\textbf{e}^{r^m}\,.
\end{equation}

\textbf{end if}

\emph{(overdamped)}: Draw $\xi^{m}$ from $\mathcal{N}(0,I_d)$:
\begin{equation}\label{eqn:update_ujnSVRGOLD}
x^{m+1}=x^m-F^m h+\sqrt{2h}\xi^{m}\,.
\end{equation}
\emph{(underdamped)}: Sample $(x^{m+1},v^{m+1})$ as Gaussian random variables with expectation and covariance defined in~\eqref{distributionofZ}, replacing $\nabla f(x^m)$ by $F^m$.

\State \textbf{end}
\State \textbf{Output:} $\{x^m\}$.
\end{algorithmic}
\end{algorithm}

\subsubsection{SAGA based variance reduction method}
SAGA is also a technique introduced in optimization to reduce the variance of SGD. It starts with a full gradient, and in each step, simply updates the gradient of one randomly selected representative function.

Inspired by the approach, we propose our Algorithm~\ref{alg:SAGA-OU-LMC}, in which we compute the full gradient only in the first round of iteration, and keep updating one randomly selected directional derivative per iteration. We term the method Random Coordinate Averaging Descent-O/U-LMC (RCAD-O/U-LMC).

\begin{algorithm}[htb]
\caption{\textbf{RCAD-O(U)-LMC}}\label{alg:SAGA-OU-LMC}
\begin{algorithmic}
\State \textbf{Preparation:}
\State 1. Input: $h$ (time stepsize); $\gamma$ (parameter); $M$ (stopping index); $d$ (dimension) and $f(x)$.
\State 2. Initial: \emph{(overdamped)}: $x^0$ i.i.d. sampled from the initial distribution induced by $q_0(x)$ and calculate $g^0\in\mathbb{R}^d$:
\begin{equation}\label{eqn:g1}
g^0_i=\partial_if(x^0),\quad \ 1\leq i\leq d\,.
\end{equation}
\emph{(underdamped)}: $(x^0,v^0)$ i.i.d. sampled from the initial distribution induced by $q_0(x,v)$ and calculate $g^0\in\mathbb{R}^d$ as \eqref{eqn:g1}.

\State \textbf{Run: }\textbf{For} $m=0\,,1\,,\cdots\,,M$

1. Draw a random number $r^m$ uniformly from $1,2,\cdots,d$.

2. Calculate $g^{m+1}$ and flux $F^m\in \mathbb{R}^d$ by letting $g^{m+1}_{i}=g^{m}_{i}$ for $i\neq r^m$ and
\begin{equation}\label{randomgradientapproximation}
g^{m+1}_{r^m}=\partial_{r^m}f(x^m)\,,\quad F^m=g^m+d\left(g^{m+1}_{r^m}-g^m_{r^m}\right)\textbf{e}^{r^m}\,.
\end{equation}

3. \emph{(overdamped)}: Draw $\xi^{m}$ from $\mathcal{N}(0,I_d)$:
\begin{equation}\label{eqn:update_ujnSASGOLD}
x^{m+1}=x^m-F^m h+\sqrt{2h}\xi^{m}\,.
\end{equation}
\emph{(underdamped)}: Sample $(x^{m+1},v^{m+1})$ as a Gaussian random variable with expectation and covariance defined in~\eqref{distributionofZ}, replacing $\nabla f(x^m)$ by $F^m$.

\State \textbf{end}
\State \textbf{Output:} $\{x^m\}$.
\end{algorithmic}
\end{algorithm}

\subsection{Convergence and numerical cost analysis}\label{sec:mainresultSAGA}
We now discuss the convergence of SVRG/RCAD-O-LMC and SVRG/RCAD-U-LMC. In the classical papers~\citep{durmus2018analysis,DALALYAN20195278,Cheng2017UnderdampedLM,dalalyan2018sampling} that discuss the convergence of O/U-LMC, the authors indeed discussed the numerical error in approximating the gradients, but both papers require that the variance of the error is bounded and independent of $d$. The random coordinate selection process is rather complicated, making their results hard to apply. One related work is \citep{pmlr-v80-chatterji18a}, where the authors construct a Lyapunov function to study the convergence of SG-MCMC. Our proof for the convergence of SVRG/RCAD-O-LMC is inspired by its technicalities. In~\citep{Cheng2017UnderdampedLM,pmlr-v80-chatterji18a}, a contraction map is used for U-LMC, but such a map cannot be directly applied in our situation because the variance depends on the entire trajectory of samples. Furthermore, the history of the trajectory is reflected in each iteration, deeming the process to be non-Markovian. We need to re-engineer the iteration formula accordingly for tracing the error propagation.

\subsubsection{Convergence for SVRG/RCAD-O-LMC}
For SVRG-O-LMC, We have the following theorem.
\begin{theorem}\label{thm:disconvergenceSVRGOLD}
Suppose $f$ satisfies Assumptions \ref{assum:Cov}-\ref{assum:Hessian} and $h$ satisfies
\begin{equation}\label{eqn:conditiononhetaSVRG}
h<\min\left\{\frac{1}{400d\kappa^2\mu },\frac{1}{10\tau\max\{\mu ,1\}}\right\}\,.
\end{equation}
Denote $W_m = W_2(q^O_m,p)$, $q^O_m$ the density of the sample $x^m$ computed from Algorithm \ref{alg:SVRG-OU-LMC} (overdamped), and $p$ the target density function, then we have
\begin{equation}\label{W2boundSVRGoldoption2}
W_m\leq \exp\left(-\frac{\mu hm}{32}\right)W_0+\left[h^{3/2}\tau dC_1+h\tau^{1/2} dC_2\right]\,,
\end{equation}
where
\[
C_1=30\kappa^{3/2}\mu ,\quad C_2=50\kappa\sqrt{\mu }+5\sqrt{\kappa^3\mu /d+2H^2/\mu ^2}\,.
\]
\end{theorem}
We leave the proof to Section \ref{proofofthm:disconvergenceSVRGOLD} and simply discuss the numerical cost here. Suppose we set $\tau = d$. And we preset the desired accuracy to be $\epsilon$, meaning we wish to obtain $W_2(q^O_m,p)\leq \epsilon$, then we can simply choose to set all terms in~\eqref{W2boundSVRGoldoption2} less than $\epsilon/3$. The latter two terms give the constraints to $h$ while the first one gives the constraint on $m$, meaning in addition to the requirement~\eqref{eqn:conditiononhetaSVRG}, we have:
\[
h\lesssim \min\left\{\frac{\epsilon^{2/3}}{d^{4/3}C^{2/3}_1}\,,\frac{\epsilon}{d^{3/2}C_2}\right\}\,,\quad\text{and}\quad mh\gtrsim \frac{1}{\mu }\log\left(\frac{W_0}{\epsilon}\right)\,.
\]
For small $\epsilon$ and large $d$, $h<\frac{\epsilon}{d^{3/2}C_2}$ is the most restrictive one, and it leads the cost $m>\widetilde{O}(d^{3/2}/\epsilon)$. Since in most iterations (unless $k\tau$) one only calculates one directional derivative, the total computational cost is $\widetilde{O}(d^{3/2}/\epsilon)$.

For RCAD-O-LMC, We have the following theorem.
\begin{theorem}\label{thm:disconvergenceSAGAOLD}[\revise{Theorem 5.1 in~\cite{ding2020variance}}]
Suppose $f$ satisfies Assumptions \ref{assum:Cov}-\ref{assum:Hessian} and $h$ satisfies
\begin{equation}\label{eqn:conditiononhetaSAGA}
h<\frac{1}{3(1+9d)\kappa^2\mu }\,.
\end{equation}
Denote $W_m=W_2(q^O_m,p)$, $q^O_m$ the density of the sample $x^m$ computed from Algorithm \ref{alg:SAGA-OU-LMC} (overdamped), and $p$ the target density function, then we have
\begin{equation}\label{iterationinequalityofSAGA}
W_m\leq \sqrt{2}\exp(-\mu hm/4)W_0+2h\sqrt{d^3C_1+d^2C_2}\,.
\end{equation}
Here $C_1=77\kappa^2\mu $, $C_2=H^2/\mu ^2+\kappa^3\mu /d$.
\end{theorem}

\revise{We omit the detailed proof for Theorem \ref{thm:disconvergenceSAGAOLD}. Interested readers are referred to the appendix of~\citep{ding2020variance}. The sketch of the proof is presented in Section \ref{proofofthm:disconvergenceSAGAOLD}.} We here discuss the computational cost. Suppose we present the desired accuracy to be $\epsilon$, then we can simply choose $h$ to set both terms in~\eqref{iterationinequalityofSAGA} less than $\epsilon/2$, and it leads to, in addition to~\eqref{eqn:conditiononhetaSAGA}
\[
h\lesssim \frac{\epsilon}{d^{3/2}\sqrt{C_1+C_2/d}}\,,\quad\text{and}\quad mh\gtrsim \frac{1}{\mu }\log\left(\frac{W_0}{\epsilon}\right)\,.
\]
For small $\epsilon$, the new constraint of $h$ dominates, and this leads to the cost $m>\widetilde{O}(d^{3/2}/\epsilon)$.

We emphasize that in Theorem \ref{thm:disconvergenceSVRGOLD} and Theorem~\ref{thm:disconvergenceSAGAOLD} we require both Assumptions~\ref{assum:Cov} and~\ref{assum:Hessian}. The second assumption is on the continuity of the hessian term. If this is relaxed, the convergence still holds true, with a degraded numerical cost. The cost will be $\widetilde{O}(\max\{d^{3/2}/\epsilon,d/\epsilon^2\})$. But compared to $\widetilde{O}(d^2/\epsilon^2)$ required by the classical O-LMC, the algorithms with reduced variance still outperform. The proof is the same, and we omit it from the paper.

\begin{remark}\label{re:5.1}
We have two comments:
\begin{itemize}
\item Compare Theorem \ref{thm:discreteconvergence}, \ref{thm:disconvergenceSVRGOLD} and \ref{thm:disconvergenceSAGAOLD}, the three theorems for overdamped LMC methods, it is rather clear that the dependence on $h$ and $m$ in the exponential is the same. Since $m$ and $h$ have the same order, then for $\epsilon$ accuracy, $m\sim \frac{1}{h}$. The $h$ requirement, however, mainly comes from the form of the second term. It is the quality of this term that determines the final numerical cost.

In particular, in Theorem \ref{thm:discreteconvergence}, $h$'s dependence on $d$ is at the power of $-2$, but in Theorem \ref{thm:disconvergenceSVRGOLD} and \ref{thm:disconvergenceSAGAOLD}, such dependence reduces to $-3/2$. This explains the improvement from RCD to SVRG and RCAD in the overdamped setting, as presented in Table 1.

\item \revise{As presented in Remark~\ref{rmk:discussionOLMC}, the major error term that needs to be improved is the second term in~\eqref{eqn:thmW2bound}, which is mainly induced by the variance term $\EE|E^m|^2$ that measures the distance between $\nabla f(x^m)$ and $F^m$. When a variance reduction technique is integrated, we can find a better estimate for $\nabla f(x^m)$, reducing the error $\EE|E^m|^2$, and improving the final convergence bound. In particular, as will be shown in Lemma~\ref{lem:SVRGOLD} and Lemma~\ref{lem:ESAGAOLD}, we roughly obtain, for SVRG-O-LMC and RCAD-O-LMC respectively,
\[
\EE|E^m|^2\leq O(h^2d^2\tau^2+hd^2\tau)\quad\text{and}\quad \EE\left|E^m\right|^2\leq O(h^2d^4+hd^3)\,.
\]
Under the condition $\tau=d$ and $h\lesssim \frac{1}{d}$, both terms, when inserted in the iteration, contribute an error term of order $O(h^3d^3)$. This is an improvement over RCD-O-LMC $(O(h^2d^2))$. More details can be found at the end of Section \ref{proofofthm:disconvergenceSVRGOLD} and \ref{proofofthm:disconvergenceSAGAOLD}.}
\end{itemize}
\end{remark}

\subsubsection{Convergence for SVRG/RCAD-U-LMC}
For SVRG-U-LMC, We have the following theorem.
\begin{theorem}\label{thm:disconvergenceSVRGULD}
Suppose $f$ satisfies Assumptions \ref{assum:Cov}, $\gamma=1/L$, and $h$ satisfies
\begin{equation}\label{eqn:conditiononhetaSVRGULD}
h\leq \min\left\{\frac{1}{1648\kappa d},\frac{1}{40\tau}\right\}\,,
\end{equation}
then, with $C_1=200\sqrt{\frac{\kappa}{\mu }}$ and $C_2=240/\sqrt{\mu}$:
\begin{equation}\label{eqn:ULDSVRGDELTAM+146}
W_m<4\exp\left(-\frac{hm}{32\kappa}\right)W_0+hd^{1/2}C_1+h^{3/2}\tau dC_2\,.
\end{equation}
Here $W_m = W_2(q^U_m,p_2)$, and $q^U_m$ denotes the density of the sample $(x^m,v^m)$ derived from Algorithm \ref{alg:SVRG-OU-LMC} (underdamped), and $p_2$ is the target density function.
\end{theorem}

We leave the proof to Section \ref{proofofthm:disconvergenceSVRGULD}. The theorem gives us the strategy of choosing parameters: to achieve $\epsilon$-accuracy, meaning to have $W_m\leq \epsilon$, we can set all terms in~\eqref{eqn:ULDSVRGDELTAM+146} less than $\epsilon/3$, and it leads to, when $\tau=d$, in addition to~\eqref{eqn:conditiononhetaSVRGULD}
\[
h\lesssim \min\left\{\frac{\epsilon}{d^{1/2}C_1},\frac{\epsilon^{2/3}}{d^{4/3}C_2^{2/3}}\right\},\quad \text{and}\quad mh \gtrsim \kappa\log\left(\frac{4W_0}{\epsilon}\right)\,.
\]
It is hard to balance the smallness of $\epsilon$ and the largeness of $d$, and we keep both restrictions, this means $m>\widetilde{O}\left(\max\left\{\frac{d^{4/3}}{\epsilon^{2/3}},\frac{d^{1/2}}{\epsilon}\right\}\right)$.

For RCAD-U-LMC, we have the following theorem:

\begin{theorem}\label{thm:thmconvergenceULDSAGA}[\revise{Theorem 5.2 in~\cite{ding2020variance}}] Assume $f(x)$ satisfies Assumption \ref{assum:Cov}, $\gamma=1/L$, and $h$ satisfies 
\begin{equation}\label{eqn:conditionuetaULDSAGA}
h\leq \frac{1}{1648\kappa d}\,,
\end{equation}
then, with $C_1=200\sqrt{\frac{\kappa}{\mu}}$ and $C_2=200/\sqrt{\mu}$:
\begin{equation}\label{eqn:convergeULDSAGA}
\begin{aligned}
W_m\leq 4\sqrt{2}\exp\left(-\frac{hm}{8\kappa}\right)W_0+hd^{1/2}C_1+h^{3/2}d^2C_2\,,
\end{aligned}
\end{equation}
where $W_m = W_2(q^U_m,p_2)$ and $q^U_m$ is the density of the sample $(x^m,v^m)$ derived from Algorithm \ref{alg:SAGA-OU-LMC} (underdamped).
\end{theorem}

\revise{The proof of this theorem is referred to~\citep{ding2020variance}, and we sketch the main strategy in Section \ref{sec:proofofthm:thmconvergenceULDSAGA}.}  Here we only discuss computational cost. To achieve $\epsilon$-accuracy, meaning to have $W_m\leq \epsilon$, we can choose all terms in~\eqref{eqn:convergeULDSAGA} less than $\epsilon/3$. This gives, in addition to~\eqref{eqn:conditionuetaULDSAGA}.
\[
h\lesssim \min\left\{\frac{\epsilon}{d^{1/2}C_1},\frac{\epsilon^{2/3}}{d^{4/3}C_2^{2/3}}\right\},\quad \text{and}\quad mh \gtrsim \kappa\log\left(\frac{4\sqrt{2}W_0}{\epsilon}\right)\,.
\]
For small $\epsilon$ and large $d$, the second term in the $h$ bound is smaller than the third term, and this means $m>\widetilde{O}\left(\max\left\{\frac{d^{4/3}}{\epsilon^{2/3}},\frac{d^{1/2}}{\epsilon}\right\}\right)$.

\begin{remark}
\revise{The improvement observed in the underdamped case is an analogy to that seen for the overdamped case. By integrating variance reduction techniques, we have a better estimate of $\nabla f(x^m)$, and this reduces the cost of $\EE|E^m|^2$. The control of this term is presented in Lemma \ref{lem:SVRGULD} and Lemma \ref{lem:ESAGAULD} respectively. More detailed discussion can be found at the end of Section \ref{proofofthm:disconvergenceSVRGULD} and \ref{sec:proofofthm:thmconvergenceULDSAGA}.}
\end{remark}

\section{Numerical result}\label{sec:Numer}

In this section, we test the efficiency of variance reduction enhanced RCD methods using the following three examples. We should mention that it is numerically challenging to compute the Wasserstein distance on a high dimensional space, so we present the convergence of error only in the weak sense, by testing it on a test function.
\begin{itemize}
\item Example 1: In this example, our target distribution is $\mathcal{N}(0,I_d)$ with $d=1000$, meaning
\[
p(x)\propto \exp\left(-\sum^d_{i=1}\frac{|x_i|^2}{2}\right)\,.
\]
The initial distributions, for the overdamped and underdamped situations respectively, are $\mathcal{N}(\textbf{0.5},I_d)$ and $\mathcal{N}(\textbf{0.5},I_{2d})$. We run RCD-O/U-LMC, RCAD-O/U-LMC and SVRG-O/U-LMC with $N=5\times 10^5$ particles with different stepsizes $h$, and test expectation estimation error by calculating
\begin{equation}\label{errortest}
\mathrm{Error}=\left|\frac{1}{N}\sum^{N}_{i=1} \phi(x^i)-\EE_p(\phi)\right|\,,
\end{equation}
where $\phi$ is the test function. In Figure \ref{Figure1}, we show the error with $\phi(x)=|x_1|^2$ using different stepsizes. In both underdamped and overdamped cases, the improvement of convergence rate is immediate when variance reduction techniques are added.
\begin{figure}[htbp]
     \centering
     \subfloat{\includegraphics[height = 0.16\textheight, width = 0.4\textwidth]{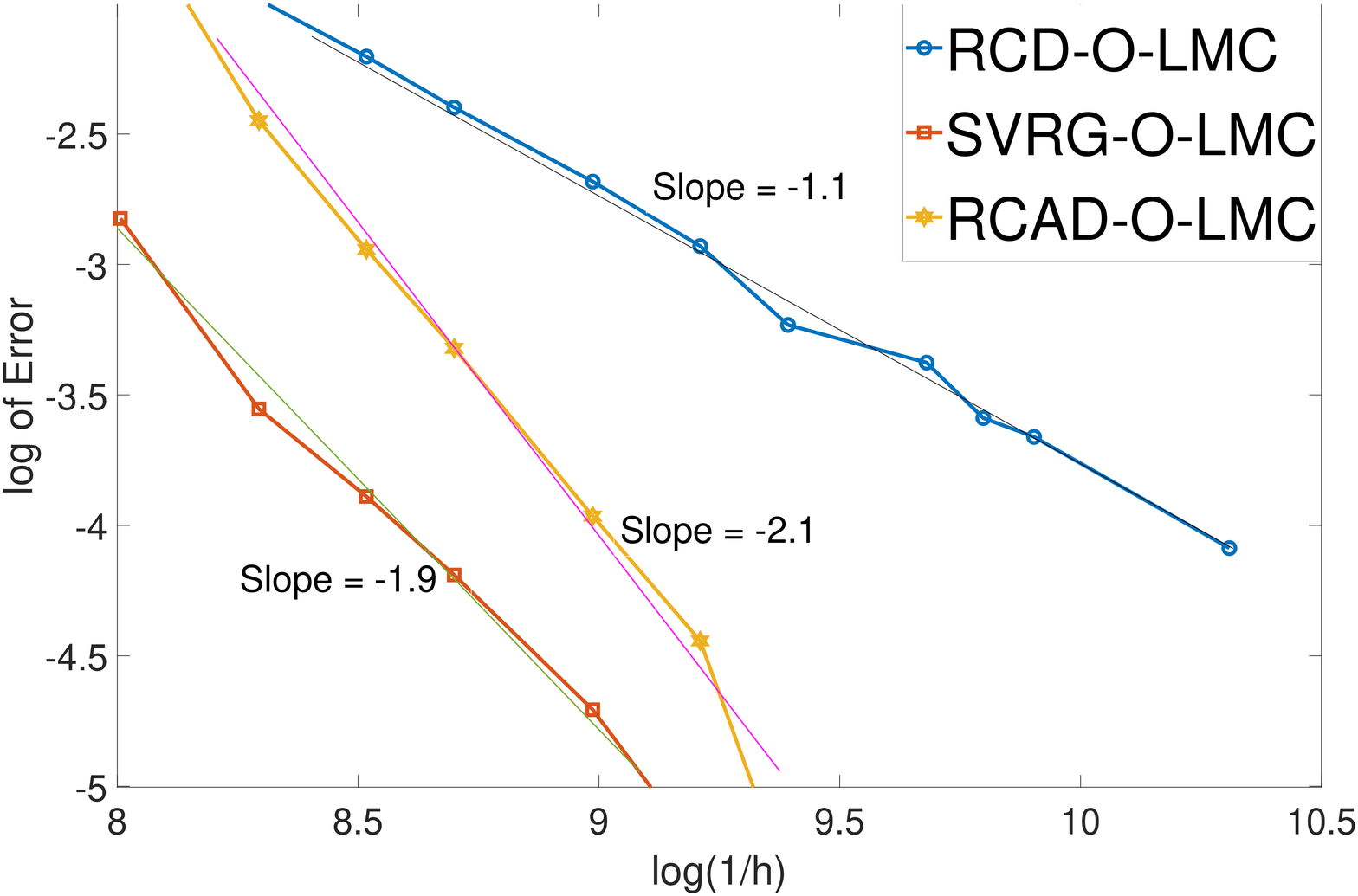}}\hspace{0.04\textwidth}
     \subfloat{\includegraphics[height = 0.16\textheight, width = 0.4\textwidth]{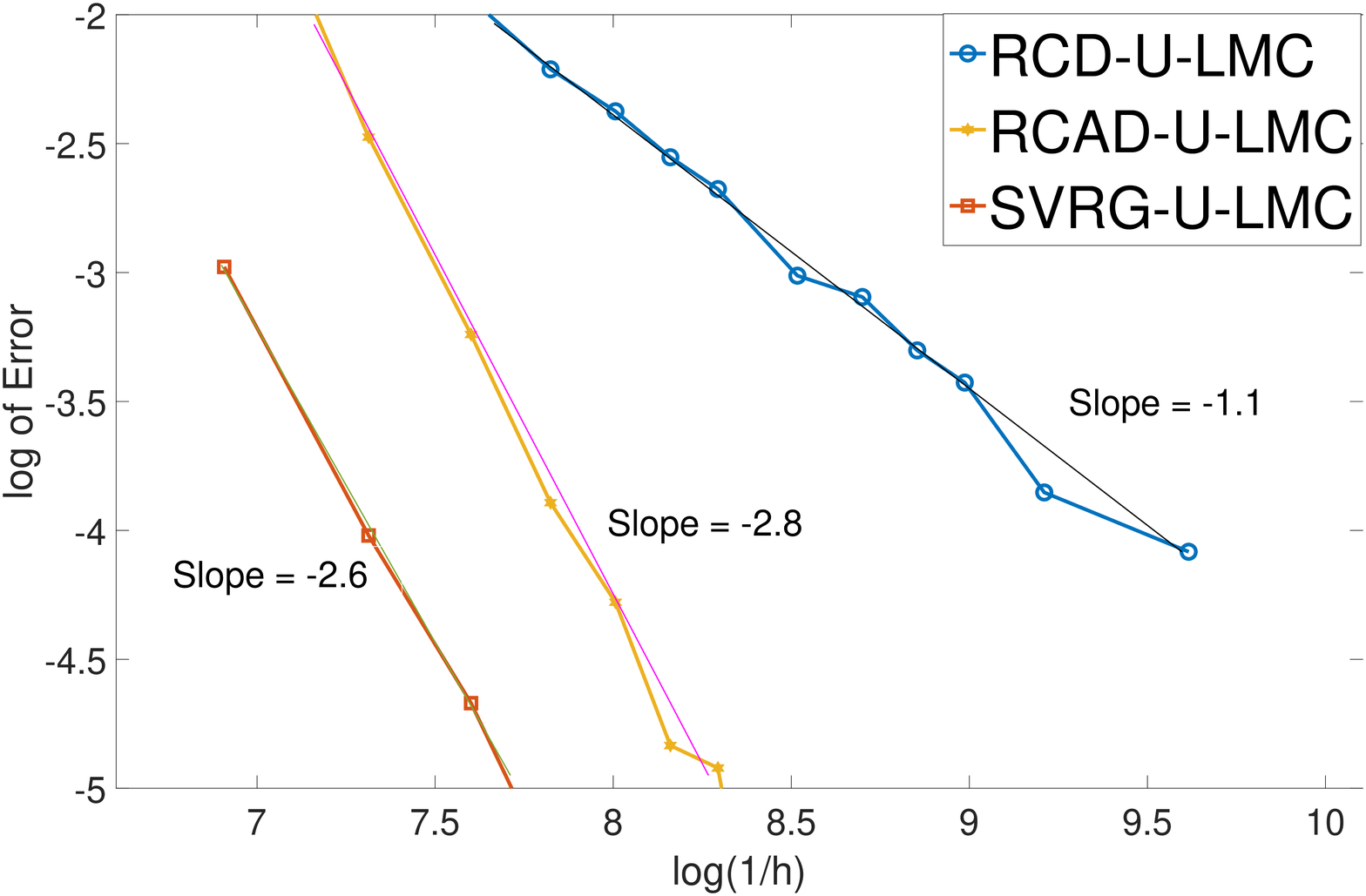}}
     \caption{Decay of Error of LMC in overdamped (left) and underdamped (right) settings.}
     \label{Figure1}
\end{figure}

\item Example 2: In this example, our target distribution is the summation of two Gaussians whose density is 
\[
p(x)\propto \exp\left(-\sum^d_{i=1}\frac{|x_i-2|^2}{2}\right)+\exp\left(-\sum^d_{i=1}\frac{|x_i+2|^2}{2}\right).
\]
This target distribution has two peaks and does not satisfy the assumptions in this paper. However, the experimental results still suggest the outperformance when variance reduction is imposed. We choose $d=1000$ and $N=10^6$ particles. The initial distributions, for the overdamped and underdamped situations respectively, are $\mathcal{N}(\textbf{0},I_d)$ and $\mathcal{N}(\textbf{0},I_{2d})$. We run RCD-O/U-LMC, RCAD-O/U-LMC, SVRG-O/U-LMC with different $h$, and calculate error using \eqref{errortest} with $\phi(x)=|x_1|^2$. In both underdamped and overdamped settings, including variance reduction techniques provide better accuracy, as shown in Figure \ref{Figure2}.
\begin{figure}[htbp]
     \centering
     \subfloat{\includegraphics[height = 0.16\textheight, width = 0.4\textwidth]{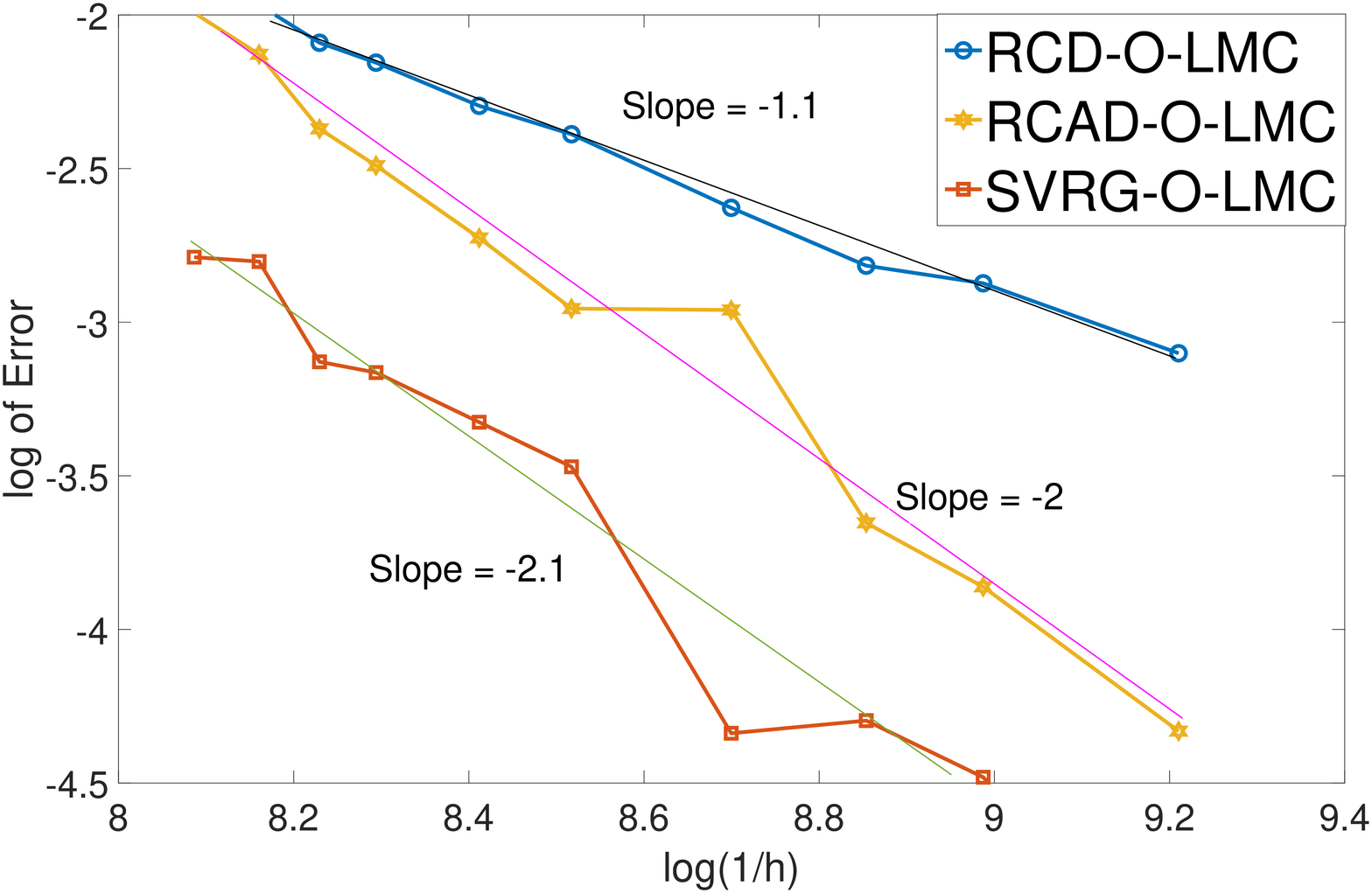}}\hspace{0.04\textwidth}
     \subfloat{\includegraphics[height = 0.16\textheight, width = 0.4\textwidth]{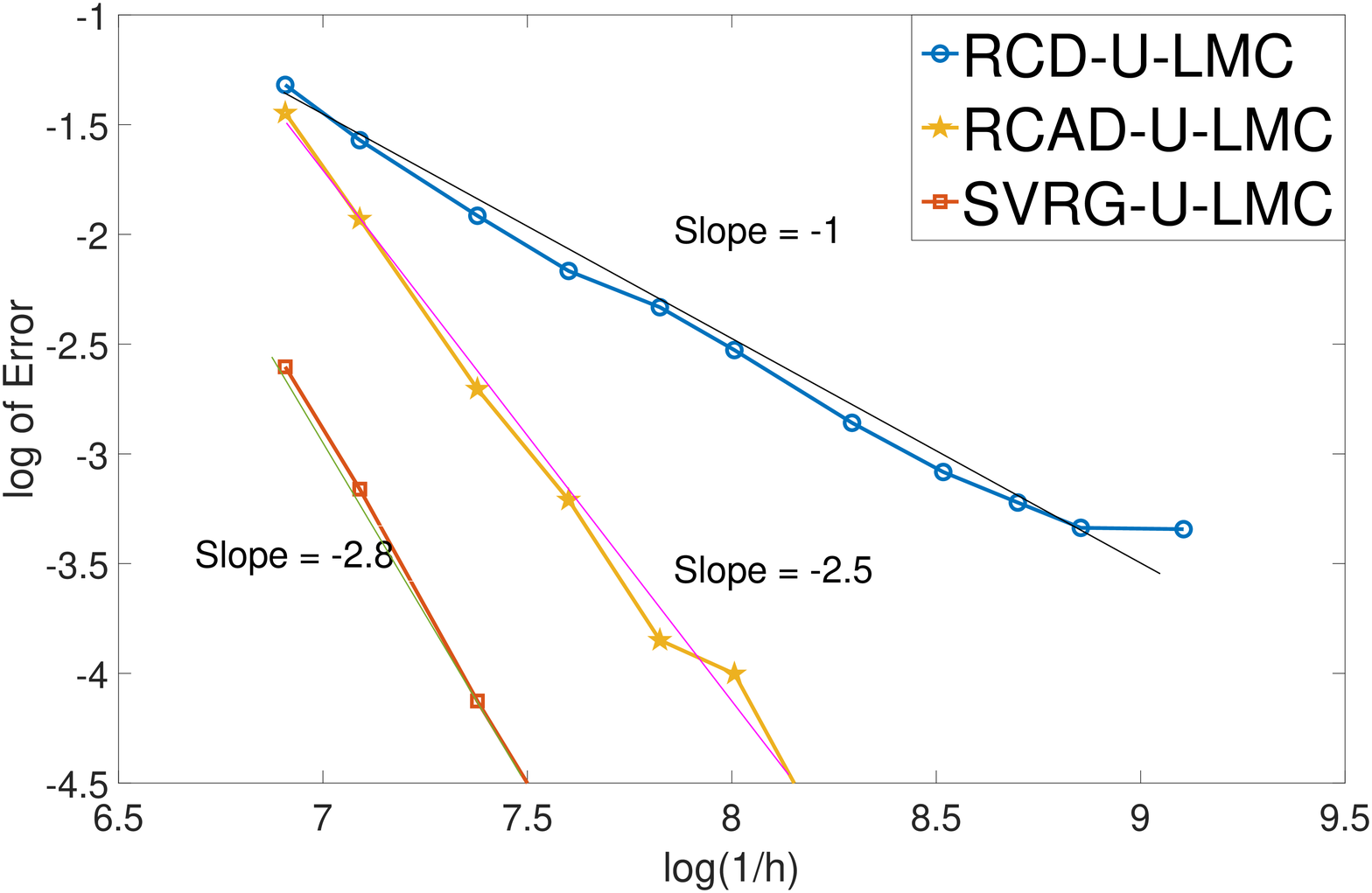}}
     \caption{Decay of Error of LMC in overdamped (left) and underdamped (right) settings.}
     \label{Figure2}
\end{figure}

\item Example 3 (Generalized linear regression) : In this example, we consider the following linear regression problem:
\begin{equation}\label{eqn:linearregression}
b=x\cdot a+\eta\,,
\end{equation}
where $a,x\in\mathbb{R}^d$ and $\eta$, the noise is a real-valued random variable satisfying $p_{\text{nos}}(\eta)$. Suppose $p_{\text{nos}}(\eta)$ is known, and one is given a set of data $\left\{(a_i,b_i)\right\}^I_{i=1}$ to infer $x$. According to the Bayes' rule, given the prior distribution of $x$, termed $p_{\text{prior}}$, the posterior distribution, also our target distribution is:
\begin{equation}\label{eqn:posterior}
p(x) = p_{\text{pos}}(x)\propto p_{\text{prior}}(x)\Pi^I_{i=1}p_{\text{nos}}(b_i-x\cdot a_i)\,.
\end{equation}
Now we look for samples sampled according to $p_{\text{pos}}$. In the test, we choose $d=100$, $x=\textbf{1}$ and generate $I=100$ data pairs. We further set the prior to be $\mathcal{N}(\textbf{0},I_d)$, and we evaluate the algorithms with two different choices of $p_{\text{nos}}$. In the first choice, this noise distribution is $\mathcal{N}(0,1)$, and in the second choice, we set:
\begin{equation}\label{eqn:etadis}
p_{\text{nos}}(\eta)\propto \exp\left(-\frac{|\eta|^2+\cos(\eta)}{2}\right)\,.
\end{equation}
All six algorithms, RCD-O/U-LMC, RCAD-O/U-LMC, SVRG-O/U-LMC are run with $N=10^6$. The error is calculated using \eqref{errortest} with $\phi(x)=|\textsf{x}|^2$, where $\textsf{x}=(x_1,x_2,\dots,x_{10})$ is the first ten components of $x$.

In Figure \ref{Figure3} and \ref{Figure4}, $\eta$ satisfies standard Gaussian and \eqref{eqn:etadis} respectively. In both cases, the variance reduction methods converge faster.

\begin{figure}[htbp]
     \centering
     \subfloat{\includegraphics[height = 0.16\textheight, width = 0.4\textwidth]{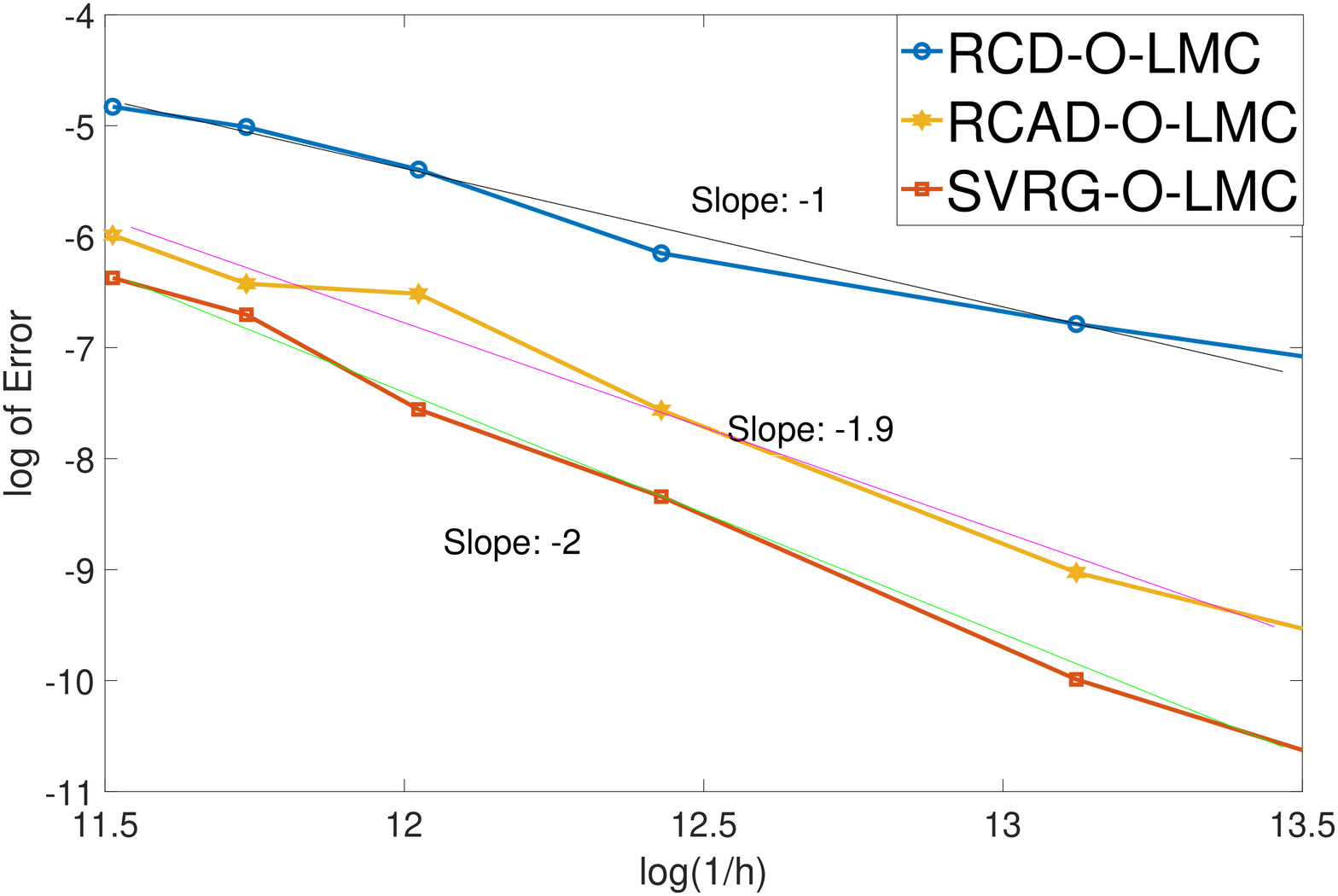}}\hspace{0.04\textwidth}
     \subfloat{\includegraphics[height = 0.16\textheight, width = 0.4\textwidth]{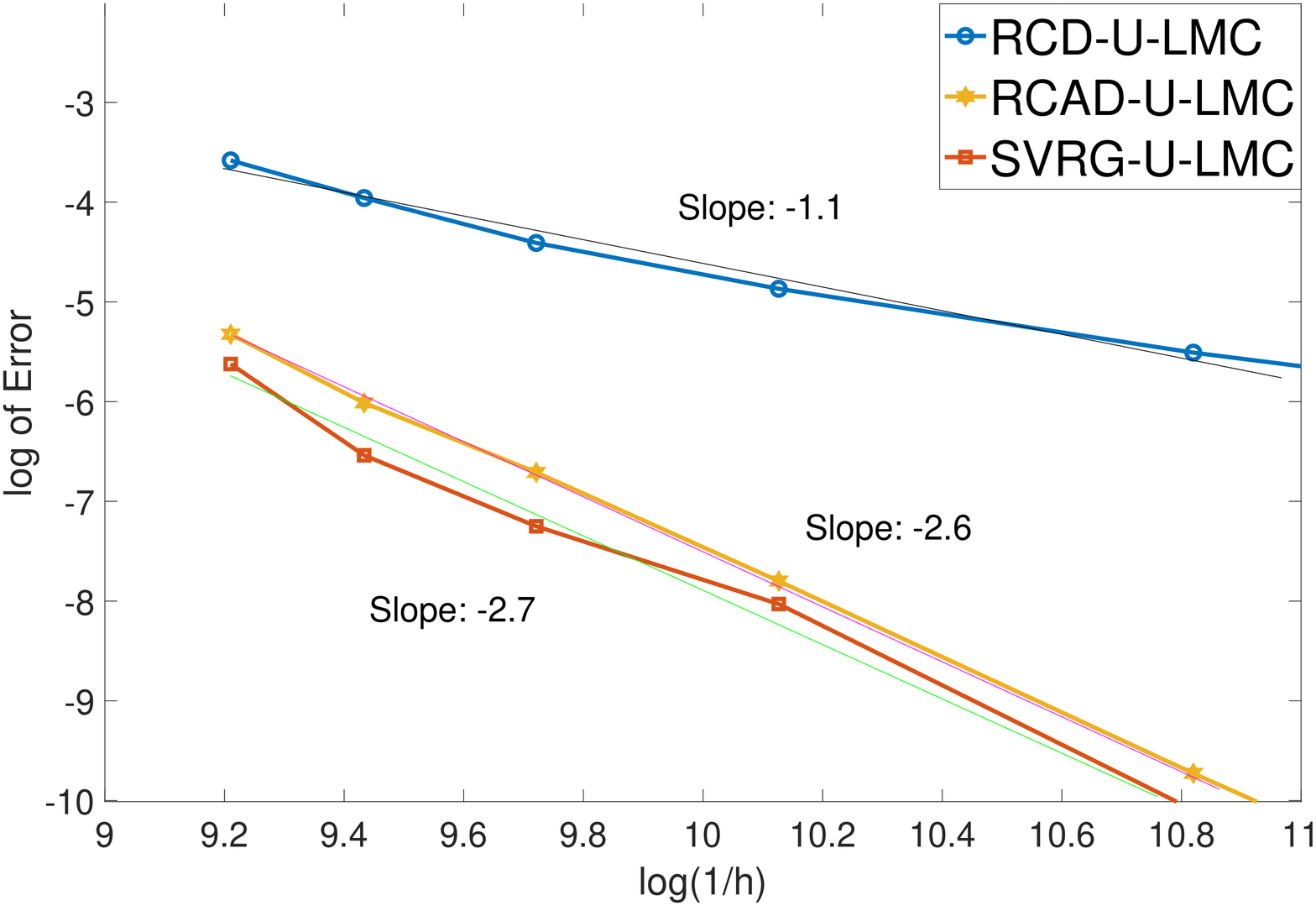}}
     \caption{Decay of Error of LMC in overdamped (left) and underdamped (right) settings.}
     \label{Figure3}
\end{figure}

\begin{figure}[htbp]
     \centering
     \subfloat{\includegraphics[height = 0.16\textheight, width = 0.4\textwidth]{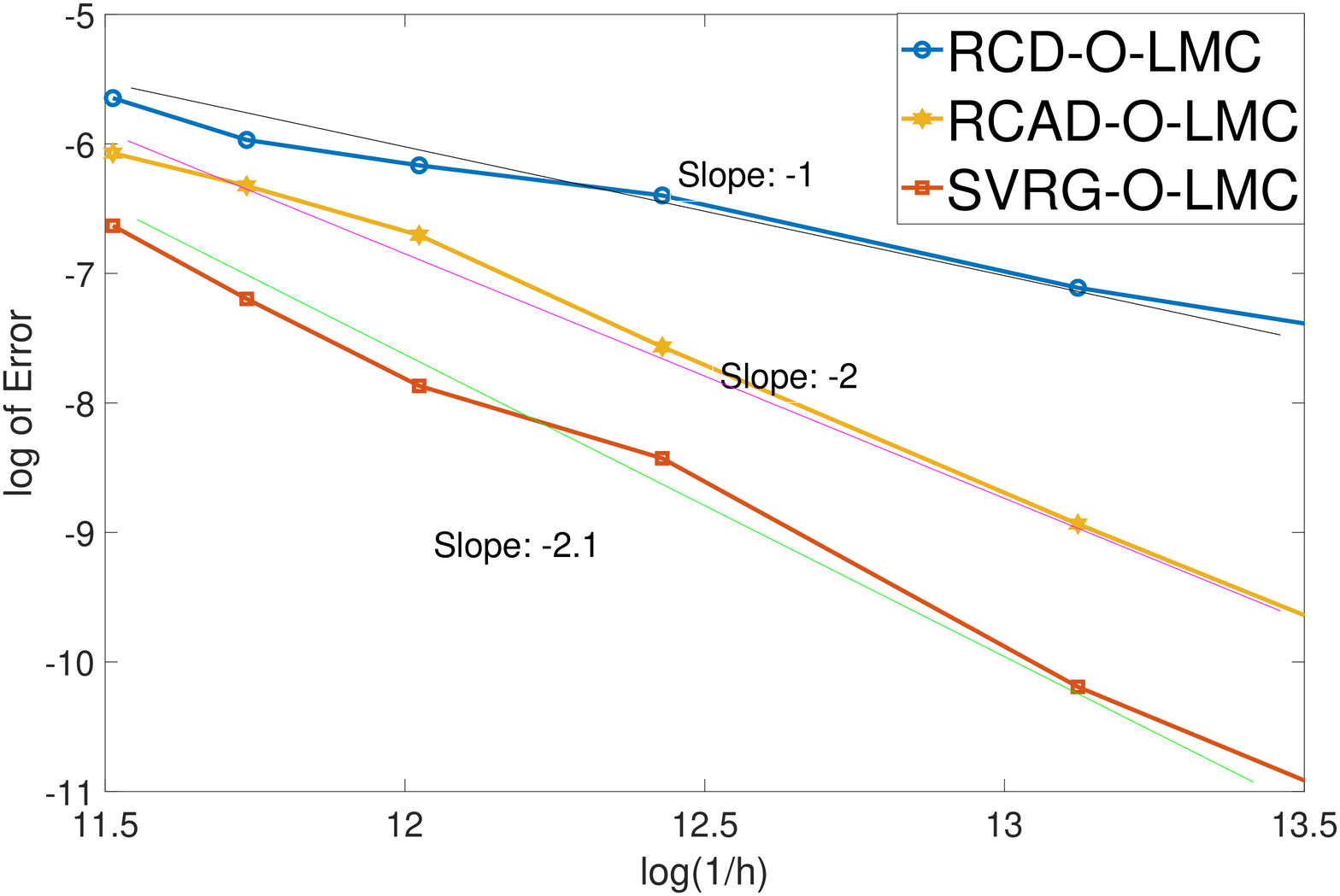}}\hspace{0.04\textwidth}
     \subfloat{\includegraphics[height = 0.16\textheight, width = 0.4\textwidth]{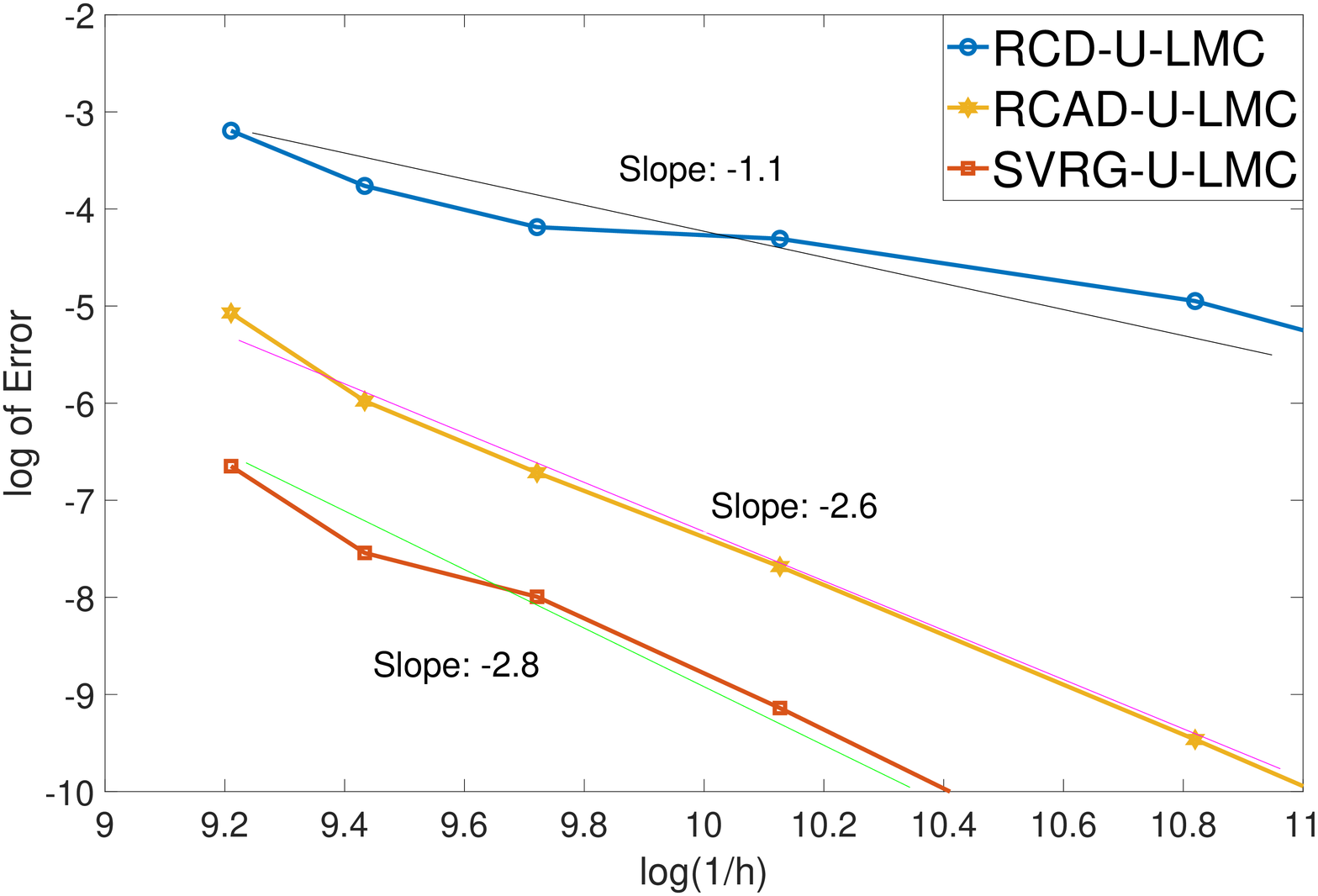}}
     \caption{Decay of Error of LMC in overdamped (left) and underdamped (right) settings.}
     \label{Figure4}
\end{figure}
\end{itemize}

In all the numerical examples above, we observe, in the O-LMC framework, the classical O-LMC gives the error saturating at $\sim h$ while the variance reduced algorithms roughly gives saturation error $\sim h^2$. In the U-LMC framework, the dependence of such error on $h$ increases from $1$ to $\sim 2.5$ when variance reduction techniques are incorporated. We should stress that we numerically observe better convergence rates than those stated in the theorems. This is potentially due to the fact that we are evaluating the error in the weak sense by testing the samples on a test function.  This is a weaker criterion than the Wasserstein distance discussed in the theorems.

\section{Proof of convergence results for O-LMC}\label{proofOLMC}
We analyze the three O-LMC methods (RCD-O-LMC, SVRG-O-LMC, and RCAD-O-LMC) together in this section.

Before diving into details, we quickly summarize the proving strategy. Recall that the target distribution $p$ is merely the equilibrium of the SDE~\eqref{eqn:Langevin}. This means, if a particle prepared at the initial stage is drawn from $p$, then following the dynamics of SDE~\eqref{eqn:Langevin}, the distribution of this particle will continue to be $p$. In the analysis below, we call the trajectory of this particle $y_t$, and the sequence generated by this particle evaluated at discrete time $y^m$. Essentially we evaluate how quickly $x^m$ converges to $y^m$ as $m$ increases. In particular, we call $\Delta^m = x^m-y^m$ and will derive an iteration formula that shows the convergence of $\Delta^m$.

In evaluating $\Delta^{m}$, two kinds of error get involved:
\begin{itemize}
\item[1.] discretization error in $h$: this amounts to controlling the discretization error of the SDE~\eqref{eqn:Langevin}. To handle this part of the error we employ the estimates in~\citep{DALALYAN20195278};
\item[2.] random coordinate selection process error: this is to measure, at each iteration, how big can $\nabla f(x^m)-F^m$ be. According to the way $F^m$ is defined, it is straightforward to show that the expectation of this error is always $0$, but the variance $\EE|\nabla f(x^m)-F^m|^2$ can be big. The analysis for the three different methods under O-LMC framework mostly concentrates on giving a bound to this term.
\end{itemize}

In a way, see details in~\eqref{eqn:Deltam+1}, we can derive the iteration formula: 
\begin{align*}
\Delta^{m+1}
=&y^{m+1} - x^{m+1}=\Delta^m+(y^{m+1}-y^{m})-(x^{m+1}-x^m)\\
=&\Delta^m-h\left(\nabla f(y^m)-\nabla f(x^m)\right)-\int^{(m+1)h}_{mh}\left(\nabla f(y_s)-\nabla f(y^m)\right)\rd s\\
&-h(\nabla f(x^m)-F^m)\,.
\end{align*}

The second term on the right-hand side, by using the Lipschitz continuity, will provide $-Lh\Delta^m$, and it produces desirable properties when combined with the first $\Delta^m$. The third term encodes the discretization error in $h$ and was shown to be small in~\citep{DALALYAN20195278}. The last term is the error that comes from the random coordinate selection process, and different methods will give different control.

It turns out that all three O-LMC methods share the same iterative formula for $\Delta^m$. We will derive this formula in Theorem~\ref{imlem:olmc}. The formula, however, requires the boundedness of a variance term $\EE|\nabla f - F|^2$, and all three methods will provide different bounds. In subsection~\ref{sec:olmc_iteration}, we discuss the background and derive the iteration formula. The subsequent three subsections are dedicated to the analysis of the variance term. When combined with Theorem~\ref{imlem:olmc}, the theorems are immediate.

\subsection{Iterative formula for $\Delta^m$}\label{sec:olmc_iteration}
We first rewrite \eqref{eqn:update_ujnSD},~\eqref{eqn:update_ujnSVRGOLD},~\eqref{eqn:update_ujnSASGOLD} into
\begin{equation}\label{eqn:SDLangevinEuler}
x^{m+1}=x^m-\nabla f(x^m)h+E^mh+\sqrt{2h}\xi^m\,, 
\end{equation}
where $E^m$ is defined in \eqref{eqn:E}:
\[
E^m=\nabla f(x^m)-F^m\,.
\]

It can be shown that in all three cases, if $r^m$ is chosen uniformly from $1\,,\cdots\,,d$:
\begin{equation}\label{expectationofEx}
\EE_{r^m}(E^m)=\textbf{0}\,.
\end{equation}
We simply need to recall the definition of $F^m$. In RCD, we have:
\begin{equation}\label{eqn:F_tilde_RCD}
F^m = d\partial_{r^m}f(x^m)\textbf{e}^{r^m}\,,
\end{equation}
and for SVRG-O-LMC it becomes:
\begin{equation}\label{eqn:F_tilde_SVRG}
F^m=\widehat{g}+d\left(g^m_{r^m}-\widehat{g}_{r^m}\right)\textbf{e}^{r^m}\,,\quad\text{with}\quad g^m_{r^m} =\partial_{r^m}f(x^m),\ \widehat{g}=\nabla f\left(x^{\left\lfloor\frac{m}{\tau}\right\rfloor}\right)\,,
\end{equation}
and for RCAD-O-LMC, it is:
\begin{equation}\label{eqn:F_tilde_RCAD}
F^m=g^m+d\left(g^{m+1}_{r^m}-g^m_{r^m}\right)\textbf{e}^{r^m}\,,\quad\text{with}\quad g^{m+1}_{r^m}=\partial_{r^m}f(x^m)\,.
\end{equation}
Taking the expectation with respect to $r^m$,~\eqref{expectationofEx} is straightforward.

We then define the dynamics of $y_t$ to be:
\begin{equation}\label{eqn:yt}
y_t=y^0-\int^t_0 \nabla f(y_s)\rd s+\sqrt{2}\int^t_0\rd B_s\,,
\end{equation}
then, calling $y^m=y_{hm}$, and letting $B_{h(m+1)}-B_{hm}=\sqrt{h}\xi^m$, we have:
\begin{equation}\label{eqn:ymolmc}
y^{m+1}=y^{m}-\int^{(m+1)h}_{mh} \nabla f(y_s)\rd s+\sqrt{2h}\xi^{m}\,.
\end{equation}

Assume $y_0$ is a random vector drawn from the target distribution induced by $p$ such that $W^2_2(q^O_0,p)=\EE|x^0-y^0|^2$, we have $y_t$ is drawn from the distribution induced by $p$ for all $t$.

Denote
\begin{equation}\label{eqn:deltaolmc}
\Delta^m=y^m-x^m\,,
\end{equation}
then
\[
W^2_2(q^O_m,p)\leq\EE|\Delta^m|^2=\EE|x^m-y^m|^2\,,
\]
where $\EE$ takes all randomness into account. We now essentially need to give a bound to $\Delta^m$ by comparing~\eqref{eqn:SDLangevinEuler} and \eqref{eqn:ymolmc}. 

In Theorem~\ref{imlem:olmc} we derive the iteration formula for updating $\Delta^m$.
\begin{theorem}\label{imlem:olmc}
Suppose $f$ satisfies Assumptions \ref{assum:Cov}-\ref{assum:Hessian},
if $\{x^m\}$ is defined in \eqref{eqn:SDLangevinEuler}, $\{y^m\}$ is defined in \eqref{eqn:ymolmc} and $\{\Delta^m\}$ comes from \eqref{eqn:deltaolmc}, then for any $a>0$ and $m\geq 0$, we have
\begin{equation}\label{firstimboundOLD}
\EE|\Delta^{m+1}|^2\leq (1+a)A\EE|\Delta^m|^2+3(1+a)h^2\EE|E^m|^2+(1+a)h^3B +\left(1+\frac{1}{a}\right)h^4C\,,
\end{equation}
where
\begin{equation}\label{eqn:ABC_general_lemma}
\begin{aligned}
A=1-2\mu h+3L^2h^2\,,\quad B = 2L^2d\,,\quad C=\frac{1}{2}\left(H^2d^2+L^3d\right)\,.
\end{aligned}
\end{equation}
\end{theorem}

The only formula this lemma requires is~\eqref{eqn:SDLangevinEuler}. There is no assumption added on $E^m$ besides $\EE_{r^m} E^m = 0$, and thus this result is applicable to all three algorithms. However, as discussed before, the variance of $E^m$ of these three methods are significantly different. The following three subsections are then dedicated to analyzing this variance term in each different method. These, together with Theorem~\ref{imlem:olmc}, finally conclude the three theorems for overdamped LMC.

\begin{remark}\label{re:6.1}
We carefully observe the relations of the four terms in~\eqref{firstimboundOLD}.
\begin{itemize}
\item The coefficient in front of $\EE|\Delta^m|^2$ is $(1+a)A$ where $A\sim 1-h$. Suppose $a$ is chosen to be small enough to make $(1+a)A<1$ strictly, then we have the hope of having exponential decay. This means $a$ should be of the order of $O(h)$.
\item Then we need to handle the remaining three terms. \revise{The last two terms come from translating the SDE~\eqref{eqn:Langevin} to O-LMC. This discretization error cannot be removed, and they are of $O(h^3d^2)$ (since $1/a\sim 1/h$ according to the previous comment). In the ideal case, we would like to have {$\EE|E^m|^2$} (the variance) term to be negligible.} If so, upon the iteration formula, one $h$ drops out, and with taking the square root, the final version of the formula should be in line with $W_m\leq \exp\left(-hm\right)W_0+dh$. This would lead to the estimate of $m=O(\frac{d}{\epsilon})$ for an $\epsilon$ accuracy. In the overdamped setting, this would be the best possible sampling scheme one could hope for when we assume the variance term does not contribute much error.
\end{itemize}
\end{remark}

\begin{proof}
We first divide $\Delta^{m+1}$ into several parts:
\begin{equation}\label{eqn:Deltam+1}
\begin{aligned}
\Delta^{m+1}=&\Delta^m+(y^{m+1}-y^{m})-(x^{m+1}-x^m)\\
=&\Delta^m+\left(-\int^{(m+1)h}_{mh}\nabla f(y_s)\rd s+\sqrt{2h}\xi_m\right)-\left(-\int^{(m+1)h}_{mh}F^m\rd s+\sqrt{2h}\xi_m\right)\\
=&\Delta^m-\left(\int^{(m+1)h}_{mh}\left(\nabla f(y_s)-F^m\right)\rd s\right)\\
=&\Delta^m-\left(\int^{(m+1)h}_{mh}\left(\nabla f(y_s)-\nabla f(y^m)+\nabla f(y^m)-\nabla f(x^m)+\nabla f(x^m)-F^m\right)\rd s\right)\\
=&\Delta^m-h\left(\nabla f(y^m)-\nabla f(x^m)\right)-\int^{(m+1)h}_{mh}\left(\nabla f(y_s)-\nabla f(y^m)\right)\rd s\\
&-h(\nabla f(x^m)-F^m)\,.
\end{aligned}
\end{equation}
Note that here $\nabla f(y^m)-\nabla f(x^m)$ can be bounded by $\Delta^m$ using the Lipschitz condition, the term $\int \left(\nabla f(y_s)-\nabla f(y^m)\right)\rd s$ reflects the numerical error accumulated in one-step of SDE computation. $\nabla f(x^m)-F^m$ is the $E^m$ term.

Denote
\[
\begin{aligned}
U^m&=\nabla f(y^m)-\nabla f(x^m)\,,\\
V^m&=\int^{(m+1)h}_{mh}\left(\nabla f(y_s)-\nabla f(y^m)-\sqrt{2}\int^s_{mh}\nabla^2f(y_r)\rd B_r\right)\rd s\,,\\
\Phi^m&=\frac{\sqrt{2}}{h}\int^{(m+1)h}_{mh}\int^s_{mh}\nabla^2f(y_r)\rd B_r\rd s\,,
\end{aligned}
\]
we now have
\begin{equation}\label{eqn:Deltam+1journal}
\begin{aligned}
\Delta^{m+1}=&\Delta^m-hU^m-(V^m+h\Phi^m)-hE^m\\
=&\Delta^m-V^m-h(U^m+\Phi^m+E^m)\,.
\end{aligned}
\end{equation}

Upon getting equation~\eqref{eqn:Deltam+1journal} it is time to analyze each term and hopefully derive an induction inequality that states $\EE|\Delta^{m+1}|^2\approx c\EE|\Delta^m|^2+d$ with $c<1$ and $d$ being of high orders of $h$, the parameters we can tune. According to Lemma 6 of~\citep{DALALYAN20195278}, we first have
\begin{equation}\label{boundforVmphim}
\EE|V^m|^2\leq \frac{h^4}{2}\left(H^2d^2+L^3d\right),\quad \EE|\Phi^m|^2\leq \frac{2L^2hd}{3}\,,
\end{equation}
and thus we have:
\begin{equation}\label{thirdDeltam+1}
\begin{aligned}
&\EE|U^m+\Phi^m+E^m|^2\\
=& 3\EE|U^m|^2+3\EE|\Phi^m|^2+3\EE|E^m|^2\,,\\
\leq& 3L^2\EE|\Delta^m|^2+2L^2hd+3\EE|E^m|^2
\end{aligned}
\end{equation}
where we used the Lipschitz continuity of $f$ for controlling $U^m$, and~\eqref{boundforVmphim} for $\Phi^m$.

We then handle the cross terms. For example, due to the independence,~\eqref{expectationofEx}, and the convexity, we have:
\begin{equation}\label{eqn:cross_UE}
\EE\left\langle\Delta^m,\Phi^m\right\rangle=0\,,\quad \EE\left\langle\Delta^m,E^m\right\rangle=0\,,\quad \left\langle\Delta^m,U^m\right\rangle \geq \mu |\Delta^m|^2\,,
\end{equation}
this means the cross term between first and the third term in the last line~\eqref{eqn:Deltam+1journal} leads to $-2Mh\EE|\Delta^m|^2$. The cross term produced by the first and the second term, however, can be hard to control, mostly because $\EE\left\langle\Delta^m,V^m\right\rangle$ is unknown. We now employ Young's inequality, meaning, for any $a>0$:

\begin{equation}\label{Deltam+1}
\begin{aligned}
\EE|\Delta^{m+1}|^2\leq (1+a)\EE|\Delta^{m+1}+V^m|^2+\left(1+\frac{1}{a}\right)\EE|V^m|^2\,.
\end{aligned}
\end{equation}
While the second term is already bounded in~\eqref{boundforVmphim}, the first term of~\eqref{Deltam+1}, according to~\eqref{eqn:Deltam+1} becomes:
\begin{equation}\label{firsttermofV}
\begin{aligned}
\EE|\Delta^{m+1}+V^m|^2=&\EE|\Delta^m-h(U^m+\Phi^m+E^m)|^2\\
=&\EE|\Delta^{m}|^2-2h\EE\left\langle\Delta^m,U^m+\Phi^m+E^m \right\rangle\\
&+h^2\EE|U^m+\Phi^m+E^m|^2\\
\leq &(1-2\mu h)\EE|\Delta^{m}|^2+h^2\EE|U^m+\Phi^m+E^m|^2
\end{aligned}\,,
\end{equation}
where we used~\eqref{eqn:cross_UE}. Plug \eqref{thirdDeltam+1} into \eqref{firsttermofV}, using the definition of the coefficients $A, B$ in~\eqref{eqn:ABC_general_lemma}:
\begin{equation}\label{firsttermofV2}
\begin{aligned}
\EE|\Delta^m-h(U^m+\Phi^m+E^m)|^2\leq A\EE|\Delta^m|^2+3h^2\EE|E^m|^2+h^3B\,.
\end{aligned}
\end{equation}
Finally, we plug~\eqref{boundforVmphim} and \eqref{firsttermofV2} into \eqref{Deltam+1} and use the definition of the coefficients $C$ in~\eqref{eqn:ABC_general_lemma}. Then we conclude~\eqref{firstimboundOLD}.
\end{proof}

\subsection{Proof of Theorem \ref{thm:discreteconvergence}}\label{proofofthm:discreteconvergence}
As discussed before, the iteration formula for $\Delta^m$ in RCD-O-LMC satisfies the formula in Theorem~\ref{imlem:olmc}, and thus the key is to give a bound to $\EE|E^m|^2$ for RCD-O-LMC. For that, we use the following lemma.
\begin{lemma}\label{lem:E}
Suppose $f$ satisfies Assumption \ref{assum:Cov},\ref{assum:Hessian}, if $\{x^m\}$ is defined in \eqref{eqn:update_ujnSD} by RCD-O-LMC, $\{y^m\}$ is defined in \eqref{eqn:ymolmc} and $\{\Delta^m\}$ comes from \eqref{eqn:deltaolmc},
then for all $m\geq 0$:
\begin{equation}\label{momentboundforA}
\EE(|E^m|^2)<2dL^2\EE|\Delta^m|^2+2d^2L\,.
\end{equation}
\end{lemma}

The proof is shown in Appendix \ref{CalculateofE}. Now we are ready to prove Theorem \ref{thm:discreteconvergence}.
\begin{proof}[Proof of Theorem \ref{thm:discreteconvergence}]
First, we prove \eqref{eqn:thmW2bound}. Plug \eqref{momentboundforA} into \eqref{firstimboundOLD}, we have
\[
\EE|\Delta^{m+1}|^2\leq (1+a)A'\EE|\Delta^m|^2+(1+a)(6Ld^2h^2+h^3B)+\left(1+\frac{1}{a}\right)h^4C\,,
\]
where $B$ and $C$ are kept untouched as in~\eqref{eqn:ABC_general_lemma} and
\[
A'=1-2\mu h+3L^2h^2+6L^2dh^2\,.
\]
Noting $A'<1-\mu h$ since
\[
h<\frac{1}{3\kappa^2\mu +6\kappa^2\mu d}=\frac{\mu }{3L^2+6L^2d}\,.
\]
To ensure the decay of $\EE|\Delta^m|^2$, we need to choose $a$ such that
$(1+a)A'=1-O(h)$. Set $a=\frac{\mu h/2}{1-\mu h}$ so that
\[
(1+a)(1-\mu h)= 1-\frac{\mu h}{2}\,,\quad\text{and}\quad 1+1/a\leq 2/\mu h\,,
\]
then we have
\begin{equation}\label{secondimboundOLD}
\EE|\Delta^{m+1}|^2\leq \left(1-\frac{\mu h}{2}\right)\EE|\Delta^m|^2+2h^2\left(6Ld^2+hB+hC/\mu \right)\,,
\end{equation}
where we use $1+a\leq 2$. Using \eqref{secondimboundOLD} iteratively, we finally have
\begin{equation}\label{lastimboundOLD}
\EE|\Delta^{m}|^2\leq \left(1-\frac{\mu h}{2}\right)^m\EE|\Delta^0|^2+32h\kappa d^2\leq \exp\left(-\frac{\mu hm}{2}\right)\EE|\Delta^0|^2+32h\kappa d^2\,,
\end{equation}
where we use
\[
hB/\mu \leq \kappa d^2,\quad hC/\mu ^2\leq \kappa d^2
\]
by {$h<\min\left\{\frac{d}{2\kappa \mu },\frac{2}{H^2/(\kappa \mu ^2)+\kappa^2 \mu /d}\right\}$}. Since 
\[
(W^O_m)^2\leq \EE|\Delta^{m}|^2,\quad \EE|\Delta^0|^2=(W^O_0)^2\,,
\]
we have \eqref{eqn:thmW2bound}.
\end{proof}

\revise{\begin{remark}\label{rmk:Emdominant1}
We comment that in the proof, the discretization error terms $B$ and $C$ contribute $O(h^3d^2)$ but the newly induced $\EE|E|^2$ term from \eqref{momentboundforA} contributes $O(h^2d^2)$, which eventually becomes the dominating term in \eqref{eqn:thmW2bound} aside from the exponential decaying term. This is the term that suggests that the number of required iterations has to be at the order of $d^2$.
\end{remark}}

\begin{remark}\label{rmk:non_uniform2}
As mentioned in Remark~\ref{rmk:non_uniform1}, using non-uniform sampling to select coordinates may not enhance the computation, when $f$ has moderate conditioning. We provide the intuition here.

When $f$ is not skewed, $\kappa=L/\mu\sim 1$. Without loss of generality, we assume $f$ attains its minimal at origin, then:
\begin{equation}\label{fproperty}
\mu |x|^2\leq |f(x)|\leq L|x|^2\Rightarrow |f(x)|\sim |x|^2\,,
\end{equation}
and for every $1\leq i\leq d$
\[
\mu |x_i|^2\leq |\partial_i f(x)|^2\leq L|x_i|^2\Rightarrow|\partial_i f(x)|^2\sim |x_i|^2\,.
\]
When $m$ is large, the distribution of $x^m$ is close to the target distribution. For simplicity, we further assume $x^m\sim p$. The using \eqref{eqn:randomfinitedifferencenonuniformRD}, we calculate the variance of $E^m$ that is defined in \eqref{eqn:E}:
\[
\EE_r\left(\left|\frac{1}{\phi_r}\left(\nabla f(x^m)\cdot \textbf{e}^{r}\right)\textbf{e}^{r}-\nabla f(x^m)\right|^2\right)=\sum^d_{i=1}\left(\frac{1}{\phi_i}-1\right)\EE|\partial_i f(x^m)|^2\sim \sum^d_{i=1}\left(\frac{1}{\phi_i}-1\right)\EE_p(|x^m_i|^2)\,.
\]

According to~\eqref{fproperty}, $\EE_p(|x_i|^2)\sim1$, and thus the derivation above becomes:
\begin{equation}\label{nonuniformvariance}
\EE\left(\left|\frac{1}{\phi_r}\left(\nabla f(x^m)\cdot \textbf{e}^{r}\right)\textbf{e}^{r}-\nabla f(x^m)\right|^2\right)\sim \sum^d_{i=1}\left(\frac{1}{\phi_i}-1\right)\,.
\end{equation}
Noting $\sum_{i=1}^d \phi_i=1$, the right-hand side of \eqref{nonuniformvariance} achieves the smallest value when $\phi_i=1/d$. This recovers the uniform RCD-LMC.

This discussion suggests that the non-uniform selection of coordinates does not bring computational efficiency to RCD-LMC when $f$ has moderate conditioning. We note that the same scenario is observed in optimization~\citep{doi:10.1137/100802001} where the outperformance of RCD over GD crucially depends on the skewness of $f$. When $f$ is skewed, then the Lipschitz constants in different directions are drastically different. We investigate the conditioning's effect on how coordinates are chosen in~\citep{ding2020random2,ding2020random}, where we show improvement can be observed if stiffer directions are chosen more frequently.
\end{remark}

\subsection{Proof of Theorem \ref{thm:disconvergenceSVRGOLD}}\label{proofofthm:disconvergenceSVRGOLD}
The strategy for showing SVRG-O-LMC is the same. We first give a bound to the variance term, and then combine the estimate with Theorem~\ref{imlem:olmc} for proving the theorem.
\begin{lemma}\label{lem:SVRGOLD}
Suppose $f$ satisfies Assumptions \ref{assum:Cov}-\ref{assum:Hessian}, if $\{x^m\}$ is defined in \eqref{eqn:update_ujnSVRGOLD} by SVRG-O-LMC, $\{y^m\}$ is defined in \eqref{eqn:ymolmc} and $\{\Delta^m\}$ comes from \eqref{eqn:deltaolmc}, then for all $m\geq 0$, define 
\[
s=\left\lfloor \frac{m}{\tau}\right\rfloor\,,
\]
then we have 
\begin{equation}\label{varianceestimationofExmSVRGOPTION2}
\EE\left|E^m\right|^2\leq 3dL^2\left[\EE\left|\Delta^m\right|^2+\EE\left|\Delta^{s\tau}\right|^2+2h^2Ld\tau^2+4hd\tau\right]\,.
\end{equation}
\end{lemma}
We put the proof in Appendix \ref{sec:proofoflem:SVRGOLD}. Now we are ready to prove Theorem \ref{thm:disconvergenceSVRGOLD}

\begin{proof}[Proof of Theorem \ref{thm:disconvergenceSVRGOLD}]
Recall
\[
s=\left\lfloor \frac{m}{\tau}\right\rfloor\,.
\]
Use \eqref{firstimboundOLD}, similar to \eqref{secondimboundOLD}, set $a=\frac{\mu h/2}{1-\mu h}$, use $h<\frac{1}{48d\kappa^2\mu }<\frac{\mu }{3L^2}$, we have 
\[
1+a=\frac{1-\mu h/2}{1-\mu h}<2,\quad 1+1/a=2/(\mu h)\,,
\]
then
\begin{equation}\label{eqn:SVRGfirstimboundOLD}
\begin{aligned}
\EE|\Delta^{m+1}|^2&\leq \left(1-\frac{\mu h}{2}\right)\EE|\Delta^m|^2+6h^2\EE|E^m|^2+2h^3B+2h^3C/\mu \\
&\leq \left(1-\frac{\mu h}{2}\right)\EE|\Delta^m|^2+6h^2\EE|E^m|^2+h^3D\,,
\end{aligned}
\end{equation}
where we set $D = 2B+2C/\mu $.

Plug \eqref{varianceestimationofExmSVRGOPTION2} into \eqref{eqn:SVRGfirstimboundOLD},
\begin{equation}\label{eqn:SVRGsecondimboundOLD}
\begin{aligned}
\EE|\Delta^{m+1}|^2\leq &\left(1-\frac{\mu h}{4}\right)\EE|\Delta^m|^2+18h^2dL^2\EE\left|\Delta^{s\tau}\right|^2\\
&+36h^4\tau^2d^2L^3+72\tau d^2h^3L^2+h^3D\,,
\end{aligned}
\end{equation}
where we use $h<\frac{1}{72d\kappa^2\mu }=\frac{\mu }{72dL^2}$. 

Use~\eqref{eqn:SVRGsecondimboundOLD} iteratively from $s\tau$ to $m$, we have
\begin{equation}\label{eqn:SVRGsecondimboundOLD2}
\begin{aligned}
\EE|\Delta^{m}|^2\leq &\left(1-\frac{\mu h}{4}\right)^{m-s\tau}\EE|\Delta^{s\tau}|^2+18h^2(m-s\tau) dL^2\EE\left|\Delta^{s\tau}\right|^2\\
&+36h^4\tau^3d^2L^3+72h^3\tau^2 d^2L^2+h^3\tau D\,.
\end{aligned}
\end{equation}

Since $\mu h\tau<10$, we have $(1-\frac{\mu h}{4})^\tau<1-\frac{\mu h\tau}{8}$, and with the other $h^2\tau dL^2$ term, this is still controlled by $1-\frac{\mu h\tau}{16}$. Therefore, choose $m=(s+1)\tau$ and use $h<\frac{\mu }{400dL^2}$, we have
\begin{equation}\label{eqn:SVRGsecondimboundOLD3}
\begin{aligned}
\EE|\Delta^{(s+1)\tau}|^2\leq &\left(1-\frac{\mu h\tau}{16}\right)\EE|\Delta^{s\tau}|^2+36h^4\tau^3d^2L^3+72h^3\tau^2 d^2L^2+h^3\tau D\\
< &\left(1-\frac{\mu h\tau}{16}\right)^{s+1}\EE|\Delta^{0}|^2+576h^3\tau^2d^2L^3/\mu +1152h^2\tau d^2L^3/\mu +16 h^2 D/\mu \,.
\end{aligned}
\end{equation}
Plug \eqref{eqn:SVRGsecondimboundOLD3} back into \eqref{eqn:SVRGsecondimboundOLD2}, use $18h^2\tau dL^2<1$, we have
\[
\begin{aligned}
\EE|\Delta^{m}|^2< &\left(\left(1-\frac{\mu h}{4}\right)^{m-s\tau}+18h^2\tau dL^2\right)\left(1-\frac{\mu h\tau}{16}\right)^{s}\EE|\Delta^{0}|^2\\
&+650h^3\tau^2d^2L^3/\mu +1250h^2\tau d^2L^2/\mu +17h^2 D/\mu \,.
\end{aligned}
\]
Use $h\tau<\frac{1}{10}$ and $18h^2(m-s\tau) dL^2<\frac{\mu h(m-s\tau)}{16}$, we have 
\[
\begin{aligned}
\left(1-\frac{\mu h}{4}\right)^{m-s\tau}+18h^2(m-s\tau) dL^2&\leq 1-\frac{\mu h(m-s\tau)}{8}+18h^2(m-s\tau) dL^2\leq 1-\frac{\mu h(m-s\tau)}{16}\\
&\leq \left(1-\frac{\mu h}{16}\right)^{m-s\tau}\,,
\end{aligned}
\]
which implies
\begin{equation}\label{eqn:SVRGsecondimboundOLD4}
\EE|\Delta^{m}|^2< \exp\left(-\frac{\mu hm}{16}\right)\EE|\Delta^{0}|^2+650h^3\tau^2d^2L^3/\mu +1250h^2\tau d^2L^2/\mu +17h^2 D/\mu \,.
\end{equation}

Taking square root on both sides and use
\[
D=2B+2C/\mu \leq d^2(4\kappa^2\mu ^2/d+\kappa^3\mu ^2/d+2H^2/\mu )\,.
\]
Since $\tau\geq1,d\geq1$, combine last two terms in \eqref{eqn:SVRGsecondimboundOLD4}, we have \eqref{W2boundSVRGoldoption2}.
\end{proof}

Similar to the comment at the end of Section \ref{proofofthm:discreteconvergence}, we here briefly review the $d$ and $h$ dependence. As is the case in Section \ref{proofofthm:discreteconvergence}, the $B$ term and the $C$ term (glued together to be called the $D$ term here) contribute $h^3$. The error induced by {$\EE|E|^2$} is bounded by \eqref{varianceestimationofExmSVRGOPTION2}, when entering \eqref{eqn:SVRGsecondimboundOLD2} is also at the order of $h^3$. These two terms combined explain the last term in \eqref{W2boundSVRGoldoption2}. However we should note that $\tau$ is typically set at the order of $d$, so the $d$ dependence from the newly introduced error in \eqref{eqn:SVRGsecondimboundOLD3} is in fact $d^3$ as a comparison of $d^2$ in $D$. This $d^3$ term determines the $m\sim d^{3/2}$ dependence in the end.

\subsection{Sketch proof of Theorem \ref{thm:disconvergenceSAGAOLD}}\label{proofofthm:disconvergenceSAGAOLD}
The same strategy as was used in the previous two subsections is used here to give an estimate to RCAD-O-LMC. We first bound the variance term $\EE|E^m|^2$.

We first define a few notations. Define $\beta^m$ using \eqref{eqn:F_tilde_RCAD} with the same $r^m$ but replacing $x^m$ with $y^m$:
\begin{equation}\label{initialbeta0}
\beta^0=\nabla f(y^0)
\end{equation}
and
\begin{equation}\label{updatebeta}
\beta^{m+1}_{r^m}=\partial_{r^m}f(y^m)\quad\text{and}\quad \beta^{m+1}_{i}=\beta^{m}_{i}\quad \text{if}\quad i\neq r^m\,.
\end{equation}
Then we have:
\begin{lemma}\label{lem:ESAGAOLD}
Suppose $f$ satisfies Assumptions \ref{assum:Cov}-\ref{assum:Hessian}, if $\{x^m\}$ is defined in \eqref{eqn:update_ujnSASGOLD} by RCAD-O-LMC, $\{y^m\}$ is defined in \eqref{eqn:ymolmc}, and $\{\Delta^m\}$ comes from \eqref{eqn:deltaolmc}, then for all $m\geq 0$, we have
\begin{equation}\label{varianceestimationofExm}
\begin{aligned}
\EE\left|E^m\right|^2\leq &3dL^2\EE|\Delta^m|^2+24hL^2d^3\left[hLd+1\right]+3d\EE\left|\beta^m-g^m\right|^2\,.
\end{aligned}
\end{equation}
\end{lemma}
The proof is seen in~\citep{ding2020variance}. One big difference between this analysis and the counterparts in the previous two subsections is that the bound of $\EE|E^m|^2$ depends on $\nabla f(y^m)-\nabla f(x^m)$, reflected in the $\beta^m-g^m$ term in~\eqref{varianceestimationofExm}. This term traces the history of the random coordinate selection, and cannot be merely bounded by $\Delta^m$, making the full iteration self-consistent. To overcome the difficulty, we define a Lyapunov function that combines the effect of $\Delta^m$ and this particular term $\beta^m-g^m$.

The Lyapunov function is defined as follows:
\begin{equation}\label{eqn:Lyna}
T^m=T^m_1+c_pT^m_2=\EE|\Delta^m|^2+c_{p}\EE|g^m-\beta^m|^2\,,
\end{equation}
where $c_{p}$ will be carefully chosen later. And the relation between $T_1$ and $T_2$ are now in the following lemma:
\begin{lemma}\label{lemma:SAGAOLD}
Under conditions of Theorem \ref{thm:disconvergenceSAGAOLD}, if $\{x^m\}$ is defined in \eqref{eqn:update_ujnSASGOLD}, $\{y^m\}$ is defined in \eqref{eqn:ymolmc} and $\{T^m_1,T^m_2\}$ comes from \eqref{eqn:Lyna}, then for any $a>0,m\geq 0$, there are upper bounds of $T^m_1$ and $T^m_2$:
\begin{equation}\label{firstimboundOLDSAGA}
T^{m+1}_1\leq (1+a)AT^m_1 + (1+a)BT^m_2 + (1+a)h^3C +\left(1+\frac{1}{a}\right)h^4D\,,
\end{equation}
and
\begin{equation}\label{secondimboundOLDSAGA}
T^{m+1}_2\leq \widetilde{A}T^m_1+\widetilde{B}T^m_2\,,
\end{equation}
where
\begin{equation*}
\begin{aligned}
&A=1-2\mu h+3(1+3d)L^2h^2\,,\quad B = 9h^2d\,,\\
&C=2L^2d+72L^2d^3\left[hLd+1\right]\,,\quad D=\frac{1}{2}\left(H^2d^2+L^3d\right)\,,\\
&\widetilde{A}=\frac{L^2}{d}\,,\quad \widetilde{B} = 1-1/d\,.
\end{aligned}
\end{equation*}
\end{lemma}

We refer interested readers to the proof in~\citep{ding2020variance}, where we present how to choose proper $a$ and $c_p$ in the lemma above to obtain Theorem \ref{thm:disconvergenceSAGAOLD}. Here, we still comment on the dependence of $h$ on $d$ again. Due to the complicated Lyapunov function, and the involvement of {$\beta^m-g^m$}, the intuition is not as straightforward as it was for Theorem \ref{thm:discreteconvergence} and \ref{thm:disconvergenceSVRGOLD}. However, according to Lemma \ref{lemma:SAGAOLD}, $\widetilde{A}$ is as small as $1/d$ while $\widetilde{B}$ is almost $1$, meaning $T^m_2$ barely changes. Moreover, $c_p$ is as small as $h^2$ and is neglected from the calculation quickly. The dominating terms, in the end, are still the $C$ term and the $D$ term in \eqref{firstimboundOLDSAGA}, both of which contribute $h^3$. The highest power of $d$ in these two terms is the $d^3$ term that comes from the last term in $C$, and this is introduced by {$\EE|E^m|^2$} in \eqref{varianceestimationofExm}. This $h^3d^3$ term, upon iteration, becomes the $h^2d^3$ term as shown in the theorem, which finally explains the $d^{3/2}$ power in the numerical cost.

As shown in Theorem \ref{thm:discreteconvergence} and \ref{thm:disconvergenceSVRGOLD}, in this theorem, we once again encounter the situation where the error of {$\EE|E^m|^2$} dominates the term in the final iteration formula.

\section{Proof of convergence results for U-LMC}\label{proofofULMC}
In this section, we prove the convergence results for RCD-U-LMC, RCAD-U-LMC, and SVRG-U-LMC.
\subsection{Iteration formula for U-LMC}\label{sec:ulmc_iteration}
Recall the definition {$E^m = \nabla f(x^m) - F^m$}, as defined in~\eqref{eqn:E}.

According to Algorithm \ref{alg:SOU-LMC}, \ref{alg:SAGA-OU-LMC}, all three algorithms can be seen as drawing $(x^0,v^0)$ from distribution induced by $q^U_0$, and
update $(x^m,v^m)$ using the following coupled SDEs:
\begin{equation}\label{eqn:ULDSDE2SAGA}
\left\{\begin{aligned}
&\mathrm{V}_t=v^me^{-2h}-\gamma \int^{t}_{mh}e^{-2(t-s)}\rd sF^m+\sqrt{4\gamma}\int^{t}_{mh}e^{-2 (t-s)}dB_s \\
&\mathrm{X}_t=x^m+\int^{t}_{mh} \mathrm{V}_sds
\end{aligned}\right.\,,
\end{equation}
where $B_s$ is the Brownian motion and $(x^{m+1},v^{m+1})=(\mathrm{X}_{(m+1)h},\mathrm{V}_{(m+1)h})$. 

We then define $w^m=x^m+v^m$, and denote $u_m(x,w)$ the distribution of $(x^m,w^m)$ and $u^\ast(x,w)$ the distribution of $(x,w)$ if $(x,v=w-x)$ is distributed according to density function $p_2$. From~\citep[Lemma 8]{Cheng2017UnderdampedLM}, we have:
\begin{equation}\label{trivialinequlaitySAGA}
 |x^m-x|^2+|v^m-v|^2\leq 4(|x^m-x|^2+|w^m-w|^2)\leq 16(|x^m-x|^2+|v^m-v|^2)\,\\
 \end{equation}
 and
 \begin{equation}\label{trivialinequlaitySAGA2}
  W^2_2(q^U_{m},p_2)\leq  4W^2_2(u_{m},u^*)\leq  16W^2_2(q^U_{m},p_2)\,.
\end{equation}

Similar to O-LMC, define another trajectory of sampling by setting  $(\widetilde{x}^0,\wv^0)$ to be drawn from distribution induced by $p_2$, and let $\wx^m=\wrx_{mh}$, $v^m=\wrv_{mh}$ be samples from $\left(\wrx_t,\wrv_t\right)$ that satisfy
\begin{equation}\label{eqn:ULDSDE2SAGAstar}
\left\{\begin{aligned}
&\wrv_t=\wv^0e^{-2 t}-\gamma \int^t_0e^{-2(t-s)}\nabla f\left(\wrx_s\right)\rd s+\sqrt{4\gamma}e^{-2 t}\int^t_0e^{2 s}dB_s \\
&\wrx_t=\wx^0+\int^t_0 \wrv_sds
\end{aligned}\right.\,
\end{equation}
with the same Brownian motion as before.

Clearly $\left(\wrx_t,\wrv_t\right)$ can be seen as drawn from distribution induced by $p_2$ for all $t$, and initially we can pick $(\wx^0,\wv^0)$ such that
\[
W^2_2(u_0,u^\ast)=\EE\left(|x^0-\wx^0|^2+|w^0-\ww^0|^2\right)\,.
\]

We then also define the error $\Delta^m$:
\begin{equation}\label{eqn:deltaulmc}
\Delta^m = \sqrt{|\wx^m-x^m|^2+|\ww^m-w^m|^2}\,.
\end{equation}

Similar to Theorem \ref{imlem:olmc}, we also have an important iteration formula for $|\Delta^m|^2$:
\begin{theorem}\label{imlem:ulmc}
Assume $f$ satisfies Assumption \ref{assum:Cov}, let $\{(x^m,v^m)\}$ defined in \eqref{eqn:ULDSDE2SAGA}, $\{(\wx^m,\wv^m)\}$ defined in \eqref{eqn:ULDSDE2SAGAstar} and $\{\Delta^m\}$ defined in \eqref{eqn:deltaulmc}, then there exists a uniform constant $D>0$ such that if $h$ satisfies
\blnote{
\[
h<\frac{1}{100(1+D)\kappa}\,,
\]}
then for any $m\geq 0$, we have\blnote{
\begin{equation}\label{eqn:imulmc}
\EE|\Delta^{m+1}|^2\leq A\EE|\Delta^{m}|^2+B\EE|E^m|^2+C\,,
\end{equation}
where
\[
A=1-h/(2\kappa)+480\kappa h^3,\quad B= 10\gamma^2h^2\,,\quad C=30h^3d/\mu\,.
\]}
\end{theorem}
\begin{proof}

We first divide $|\Delta^{m+1}|^2$ into different parts , and compare \eqref{eqn:ULDSDE2SAGA} and \eqref{eqn:ULDSDE2SAGAstar} for:
\[
\begin{aligned}
|\Delta^{m+1}|^2=&\left|
(\wv^m-v^m)e^{-2h}+(\wx^m-x^m)+\int^{(m+1)h}_{mh}\wrv_s-\rv_s\rd s\right.\\
&-\gamma\int^{(m+1)h}_{mh}e^{-2((m+1)h-s)}\left[\nabla f\left(\wrx_s\right)-\nabla f(x^m)\right]\rd s\\
&\left.+\gamma\int^{(m+1)h}_{mh}e^{-2((m+1)h-s)}E^m\rd s
\right|^2\\
&+\left|
(\wx^m-x^m)+\int^{(m+1)h}_{mh}\wrv_s-\rv_s\rd s
\right|^2\\
=&\left|\mathrm{J}^m_1\right|^2+\left|\mathrm{J}^m_2\right|^2=\left|\mathrm{J}^{r,m}_1+\mathrm{J}^{E,m}_1\right|^2+\left|\mathrm{J}^m_2\right|^2\,,
\end{aligned}
\]
where we denote
\begin{equation}\label{eqn:J1rm}
\begin{aligned}
\mathrm{J}^{r,m}_1&=(\wv^m-v^m)e^{-2h}+(\wx^m-x^m)+\int^{(m+1)h}_{mh}\wrv_s-\rv_s\rd s\\
&-\gamma\int^{(m+1)h}_{mh}e^{-2((m+1)h-s)}\left[\nabla f\left(\wrx_s\right)-\nabla f(x^m)\right]\rd s\,
\end{aligned}
\end{equation}
and
\begin{equation}\label{eqn:J2rm}
\begin{aligned}
\mathrm{J}^{E,m}_1=\gamma\int^{(m+1)h}_{mh}e^{-2((m+1)h-s)}E^m\rd s\,,\quad\mathrm{J}^m_2=(\wx^m-x^m)+\int^{(m+1)h}_{mh}\wrv_s-\rv_s\rd s\,.
\end{aligned}
\end{equation}
~\blnote{
To control $\mathrm{J}^m_1$, we realize all terms in $\mathrm{J}^{r,m}_1$, except $\rv_s$, are independent of $r^m$, and thus,
\begin{equation}\label{eqn:rmindependence}
\EE\langle \mathrm{J}^{r,m}_1,\mathrm{J}^{E,m}_1\rangle = -\EE\left(\EE_{r^m}\left\langle \int^{(m+1)h}_{mh}\rv_s\rd s,\mathrm{J}^{E,m}_1\right\rangle\right)\,,
\end{equation}
where all other terms are eliminated during the process of taking $\EE_{r^m}$. Therefore, we can bound $\EE\left|\mathrm{J}^m_1\right|^2$ by
\[
\begin{aligned}
&\EE\left|\mathrm{J}^m_1\right|^2=\EE \left|\mathrm{J}^{r,m}_1+\mathrm{J}^{E,m}_1\right|^2\\
=&\EE\left|\mathrm{J}^{r,m}_1\right|^2+\EE\left|\mathrm{J}^{E,m}_1\right|^2+2\EE\left\langle\mathrm{J}^{r,m}_1,\mathrm{J}^{E,m}_1\right\rangle\\
\leq&\EE\left|\mathrm{J}^{r,m}_1\right|^2+ \gamma^2h^2\EE\left|E^m\right|^2 + 2\gamma^2h^3\EE|E^m|^2\,,
\end{aligned}
\]
where we used
\[
\EE\left|\mathrm{J}^{E,m}_1\right|^2\leq \gamma^2h^2\EE\left|E^m\right|^2\,
\]
and
\[
\begin{aligned}
&2\EE\langle \mathrm{J}^{r,m}_1,\mathrm{J}^{E,m}_1\rangle \\
=&-2\EE\left\langle\int^{(m+1)h}_{mh}\rv_s\rd s,\gamma\int^{(m+1)h}_{mh}e^{-2((m+1)h-s)}\rd sE^m\right\rangle\\
=&2\EE\left\langle\gamma\int^{(m+1)h}_{mh}\int^s_{mh} e^{-2(s-t)}\rd t\rd sE^m,\gamma\int^{(m+1)h}_{mh}e^{-2((m+1)h-s)}\rd sE^m\right\rangle\\
\leq&2\gamma^2h^3\EE|E^m|^2
\end{aligned}\,,
\]
In conclusion, we have
\begin{equation}\label{Deltam+1ULD}
\begin{aligned}
\EE\left|\Delta^{m+1}\right|^2\leq &\EE\left|\mathrm{J}^{r,m}_1\right|^2+\left|\mathrm{J}^m_2\right|^2+\left(\gamma^2h^2\EE|E^m|^2+2\gamma^2h^3\EE|E^m|^2\right)\,,
\end{aligned}
\end{equation}}
which implies
\begin{equation}\label{eqn:ULDSAGADELTAM+1}
\begin{aligned}
\EE\left|\Delta^{m+1}\right|^2\leq &\EE\left|\mathrm{J}^{r,m}_1\right|^2+\left|\mathrm{J}^m_2\right|^2+2\gamma^2(h^2+3h^3)\EE|E^m|^2\,.
\end{aligned}
\end{equation}
According to Proposition \ref{prop:boundJ1m} \eqref{JRM1JM2}, for any $a>0$, we have
\[
\EE\left|\mathrm{J}^{r,m}_1\right|^2+\left|\mathrm{J}^m_2\right|^2\leq C_1\EE|\Delta^m|^2+5(1+1/a)\gamma^2h^4\EE(|E^m|^2)+5(1+1/a)\gamma h^4d\,,
\]
where
\[
C_1=(1+a)[1-h/\kappa+200h^2]+80(1+1/a)h^4\,.
\]
Plug into \eqref{eqn:ULDSAGADELTAM+1}, we have
\begin{equation}\label{JRM1JM2next}
\begin{aligned}
\EE|\Delta^m|^2\leq &C_1\EE|\Delta^m|^2+C_2\EE|E^m|^2+C_3\,,
\end{aligned}
\end{equation}
where we use $\gamma L=1,h<1$ for:
\[
C_2=\gamma^2\left[5(1+1/a)h^4+8h^2\right]\,,\quad C_3 = 5(1+1/a)\gamma h^4d\,.
\]
Considering $h\leq \frac{1}{600\kappa}$, we have
\[
1-h/\kappa+200h^2\leq 1-2h/(3\kappa)\,,
\]
and thus by setting $a$ so that
\[
1+a=\frac{1-h/(2\kappa)}{1-2h/(3\kappa)}\,,\quad\text{which leads}\quad 1+1/a\leq 6\kappa/h\,,
\]
we have
\[
C_1\leq 1-h/(2\kappa)+480\kappa h^3\,,\quad C_2\leq 30\gamma^2\kappa h^3+8\gamma^2h^2\,,\quad C_3\leq 30\kappa\gamma h^3d\,.
\]
Noting $h<\frac{1}{600\kappa}$, we have $30\kappa h^3<2h^2$, making:
\[
C_1\leq A\,,\quad C_2\leq B= 10\gamma^2h^2\,,\quad C_3\leq C=30h^3d/\mu\,,
\]
finishing the proof.
\end{proof}

Similar to before we need to make a list of comments:
\begin{remark}\label{re:7.1}
Several comments are in order.
\begin{itemize}
\item The proof is completely independent of \blnote{$\EE|E^m|^2$}, meaning the lemma holds true for all U-LMC. The differences between different sampling methods come in through the \blnote{$\EE|E^m|^2$} term that will be analyzed in the following subsections respectively.
\item \revise{Similar to O-LMC, the last term of \eqref{eqn:imulmc} corresponds to the discretization error of U-LMC.} In the ideal case, let us suppose \blnote{$\EE|E^m|^2$} is negligible, then besides the iteration term ($A$ term), the remainder term $C$ is at the order of $h^3d$. Upon iterations, this will finally lead to a formula similar to $W_m< \exp(-hm)W_0 + hd^{1/2}$. With the same argument as before, we have $h<\frac{\epsilon}{d^{1/2}}$, and $m>1/h$ for an $\epsilon$ accuracy. This is the optimal result one could hope for in the underdamped LMC case. As will be shown below, the \blnote{$\EE|E^m|^2$} is a dominating term for RCD-U-LMC, and such optimal rate is not achieved. However, in the other two cases, \blnote{$\EE|E^m|^2$} is not a dominating term if $\epsilon$ is small enough, this means a rate the same as U-LMC is possible for small $\epsilon$.
\end{itemize}
\end{remark}

\subsection{Proof of Theorem \ref{thm:discreteconvergenceULD}}\label{sec:proofofthmdiscreteconvergenceULD}
To prove Theorem~\ref{thm:discreteconvergenceULD}, we note that the definition of \blnote{$E$} is the same as in~\eqref{eqn:E}, and the results in Lemma \ref{lem:E} can be recycled:
\begin{proof}[Proof of Theorem \ref{thm:discreteconvergenceULD}]
Recall Lemma \ref{lem:E}, we have:
\blnote{
\begin{equation}\label{momentboundforAulc}
\EE(|E^m|^2)<2dL^2\EE|\Delta^m|^2+2L^2d^2/\mu \,.
\end{equation}
\\
Plug \eqref{momentboundforAulc} into \eqref{eqn:imulmc}, we have
\begin{equation}\label{secondimbounduld}
\begin{aligned}
\EE|\Delta^{m+1}|^2&\leq (A+2dL^2B)\EE|\Delta^{m}|^2+2L^2d^2B/\mu +C\\
&\leq C_1\EE|\Delta^{m}|^2+2L^2d^2B/\mu +C\,.
\end{aligned}
\end{equation}
Recalling
\[
A=1-h/(2\kappa)+480\kappa h^3,\quad B=10\gamma^2h^2\,,\quad C=30h^3d/\mu\,,
\]
we have, due to $\gamma L = 1$:
\[
C_1 = 1-h/(2\kappa)+480\kappa h^3 +20dh^2\,.
\]
For $h$ small enough $\left(h<\frac{1}{880d\kappa}\right)$, it is clear that
\begin{equation}\label{eqn:bound_C_1_B_C}
C_1< 1-h/(4\kappa)\,.
\end{equation}
Plug \eqref{eqn:bound_C_1_B_C} into \eqref{secondimbounduld}, we have
\begin{equation}\label{finalimbounduld}
\EE|\Delta^{m+1}|^2 \leq \left(1-h/(4\kappa)\right)\EE|\Delta^{m}|^2+20d^2h^2/\mu +30h^3d/\mu\,,
\end{equation}
which implies
\[
\begin{aligned}
\EE|\Delta^{m}|^2 &\leq \exp\left(-hm/(4\kappa)\right)\EE|\Delta^{0}|^2+80\kappa d^2h/\mu +120\kappa h^2d/\mu\\
&\leq \exp\left(-hm/(4\kappa)\right)\EE|\Delta^{0}|^2+100\kappa d^2h/\mu
\end{aligned}
\]
after iteration. Taking square root on each term and using \eqref{trivialinequlaitySAGA2}, we prove \eqref{eqn:thmW2boundULD}.}
\end{proof}

\revise{\begin{remark}\label{rmk:Emdominant2}
The result comes with no surprise. The discretization error term $C$ contributes $O(h^3d)$ but the newly induced $\EE|E|^2$ term from \eqref{momentboundforA} contributes $O(h^2d^2)$ .
This $h^2d^2$ term, upon iteration, drops down to $hd^2$, which explains the $d^2$ dependence in the end.
\end{remark}}

\subsection{Proof of Theorem \ref{thm:disconvergenceSVRGULD}}\label{proofofthm:disconvergenceSVRGULD}
We first have the lemma bounding \blnote{$\EE|E^m|^2$}:
\begin{lemma}\label{lem:SVRGULD} Assume $f$ satisfies Assumption \ref{assum:Cov}, and let $\{(x^m,v^m)\}$ be defined in algorithm \ref{alg:SVRG-OU-LMC} (underdamped), $\{(\wx^m,\wv^m)\}$ be defined in \eqref{eqn:ULDSDE2SAGAstar}, then for any $m\geq0$, call
\[
s=\left\lfloor \frac{m}{\tau}\right\rfloor\,,
\]
we have\blnote{ 
\begin{equation}\label{varianceestimationofExmSVRGOPTION2ULD}
\EE\left|E^m\right|^2\leq 3dL^2\left[\EE\left|\Delta^m\right|^2+\EE\left|\Delta^{s\tau}\right|^2+2\tau^2h^2\gamma d\right]\,.
\end{equation}}
\end{lemma}
We put the proof in Appendix \ref{sec:proofoflem:SVRGULD}. Now, we are ready to prove the Theorem \ref{thm:disconvergenceSVRGULD}.

\begin{proof}[Proof of Theorem \ref{thm:disconvergenceSVRGULD}]
Use Theorem~\ref{imlem:ulmc} \eqref{eqn:imulmc}, we first bound $\EE|E^m|^2$ by plugging in Lemma \ref{lem:SVRGULD} equation \eqref{varianceestimationofExmSVRGOPTION2ULD},
\[
\EE|\Delta^{m+1}|^2\leq (A+3dL^2B)\EE|\Delta^{m}|^2+3dL^2B\EE|\Delta^{s\tau}|^2+6Ld^2\tau^2h^2B+C\,.
\]
Recalling \blnote{$A=1-h/(2\kappa)+480\kappa h^3$, $B=10\gamma^2h^2$ and $C=30h^3d/\mu$}, we have
\[
A+3dL^2B\leq 1-h/(4\kappa)\,,\quad\text{when}\quad h<\frac{1}{1648 d\kappa}
\]
and thus
\blnote{
\begin{equation}\label{eqn:ULDSVRGDELTAM+14}
\begin{aligned}
\EE|\Delta^{m+1}|^2&\leq \left(1-\frac{h}{4\kappa}\right)\EE|\Delta^m|^2+30h^2d\EE|\Delta^{s\tau}|^2+30h^3d/\mu+60\gamma h^4\tau^2d^2\,.
\end{aligned}
\end{equation}}
Use $\gamma = \frac{1}{L}=\frac{1}{\kappa \mu}$, we have:
\blnote{
\begin{equation}\label{eqn:ULDSVRGDELTAM+144}
\begin{aligned}
\EE|\Delta^{m}|^2< &\left(1-\frac{h}{4\kappa }\right)^{m-s\tau}\EE|\Delta^{s\tau}|^2+30h^2(m-s\tau)d\EE|\Delta^{s\tau}|^2+30h^3\tau d/\mu+60h^4\tau^2d^2/(\kappa \mu )\,.
\end{aligned}
\end{equation}}
Similar to \eqref{eqn:SVRGsecondimboundOLD3}, by setting $m = (s+1)\tau$, we also have:\blnote{
\begin{equation}\label{eqn:ULDSVRGDELTAM+142}
\begin{aligned}
\EE|\Delta^{(s+1)\tau}|^2< &\left(1-\frac{h\tau}{8\kappa }\right)\EE|\Delta^{s\tau}|^2+30h^2\tau d\EE|\Delta^{s\tau}|^2+30h^3\tau d/\mu+60h^4\tau^3d^2/(\kappa \mu )\\
< &\left(1-\frac{h\tau}{16\kappa }\right)\EE|\Delta^{s\tau}|^2+30h^3\tau d/\mu+60h^4\tau^3d^2/(\kappa \mu )\,,\\
\end{aligned}
\end{equation}}
where we absorb the \blnote{$30h^2\tau d$} term into $\frac{h\tau}{16\kappa }$ for small enough $h$ (for example when $h<\frac{1}{1648d\kappa}$) in the second inequality. This finally gives:\blnote{
\begin{equation}\label{eqn:ULDSVRGDELTAM+143}
\EE|\Delta^{s\tau}|^2<\left(1-\frac{h\tau}{16\kappa }\right)^s\EE|\Delta^{0}|^2+480h^2d\kappa /\mu+960h^3\tau^2d^2/\mu \,.
\end{equation}}

Plug \eqref{eqn:ULDSVRGDELTAM+143} back into \eqref{eqn:ULDSVRGDELTAM+144}, we have:\blnote{
\[
\begin{aligned}
\EE|\Delta^{m}|^2<&\left(\left(1-\frac{h}{4\kappa }\right)^{m-s\tau}+30h^2(m-s\tau) d\right)\left(1-\frac{h\tau}{16\kappa}\right)^s\EE|\Delta^{0}|^2+510h^2d\kappa/\mu +1020h^3\tau^2d^2/\mu\\
<&\left(1-\frac{h}{16\kappa }\right)^{m}\EE|\Delta^{0}|^2+510h^2d\kappa/\mu +1020h^3\tau^2d^2/\mu \\
<&\exp\left(-\frac{hm}{16\kappa}\right)\EE|\Delta^{0}|^2+510h^2d\kappa/\mu +1020h^3\tau^2d^2/\mu\,,
\end{aligned}
\]
where in the first inequality we used $h\tau < 1\leq\kappa$, and in the second inequality we used $h(m-s\tau)\leq h\tau<\frac{1}{40}$ and $30h^2(m-s\tau) d<1-\frac{(m-s\tau)h}{16\kappa}$ to obtain
\[
\begin{aligned}
\left(1-\frac{h}{4\kappa }\right)^{m-s\tau}+30h^2(m-s\tau) d&<1-\frac{(m-s\tau)h}{8\kappa}+30h^2(m-s\tau) d<1-\frac{(m-s\tau)h}{16\kappa}\\
&\leq \left(1-\frac{h}{16\kappa }\right)^{m-s\tau}\,.
\end{aligned}
\]
Taking square root on both sides gives the conclusion \eqref{eqn:ULDSVRGDELTAM+146}.}
\end{proof}

\revise{
Again, we discuss the newly introduced error by \blnote{$\EE|E^m|^2$}. Take $\tau = d$ in Lemma \ref{lem:SVRGULD}, this part of error is at the order of $h^4d^4$ when entering the iteration formula with an extra $h^2$ term from $B$. By induction, one $h$ drops out and upon taking square root, this part of error is at the order of $h^{3/2}d^2$. Making it smaller than $\epsilon$ gives $h<\frac{\epsilon^{2/3}}{d^{4/3}}$. When $\epsilon$ is small enough, this is a more relaxed constraint than the original discretization constraint $h<\frac{\epsilon^{1/2}}{d}$.}

\subsection{Sketch proof of Theorem \ref{thm:thmconvergenceULDSAGA}}\label{sec:proofofthm:thmconvergenceULDSAGA}
Similar to section \ref{proofofthm:disconvergenceSAGAOLD}, we have the following estimation for \blnote{$\EE\left|E^m\right|^2$:}
\begin{lemma}\label{lem:ESAGAULD}
Assume $f$ satisfies Assumption \ref{assum:Cov}, if $\{(x^m,v^m)\}$ is defined in algorithm \ref{alg:SAGA-OU-LMC} (underdamped), $\{(\wx^m,\wv^m)\}$ is defined in \eqref{eqn:ULDSDE2SAGAstar}, then for any $m\geq0$, we have\blnote{
\begin{equation}\label{eqn:varianceofEULMC}
\EE(|E^m|^2)\leq 3dL^2\EE|\Delta^m|^2+24Lh^2d^4+3d\EE\left|\beta^m-g^m\right|^2\,.
\end{equation}}
\end{lemma}
We leave out the proof and refer interested readers to~\citep{ding2020variance}. Here, \blnote{$\beta^m$ and $g$ are defined same as \eqref{updatebeta} and \eqref{eqn:F_tilde_RCAD}} ($y$ change to $\widetilde{x}$) and we will show the decay of the following Lyapunov function:\blnote{
\begin{equation}\label{eqn:LynaULMC}
T^m\triangleq T^m_1+c_pT^m_2=\EE\left(|\wx^m-x^m|^2+|\ww^m-w^m|^2\right)+c_{p}\EE|g^m-\beta^m|^2\,,
\end{equation}}
where $c_{p}$ will be carefully chosen later. And the relation between $T_1$ and $T_2$ are now in the following lemma:
\begin{lemma} Under conditions of Theorem \ref{thm:thmconvergenceULDSAGA}, if $\{(x^m,v^m)\}$ is defined in algorithm \ref{alg:SAGA-OU-LMC} (underdamped), $\{(\wx^m,\wv^m)\}$ is defined in \eqref{eqn:ULDSDE2SAGAstar} and $\{(T^m_1,T^m_2)\}$ comes from \eqref{eqn:LynaULMC}, then for any $m\geq0$, we have
\begin{equation}\label{TM+11BOUNDULD}
\begin{aligned}
T^{m+1}_1< &D_1T^{m}_1+D_2T^{m}_2+D_3\,,
\end{aligned}
\end{equation}
\begin{equation}\label{TM+12BOUNDULD}
T^{m+1}_2\leq \frac{L^2}{d}T^m_1+\left(1-\frac{1}{d}\right)T^m_2\,,
\end{equation}
where\blnote{
\[
D_1=1-h/(2\kappa)+40h^2d,\ D_2=30\gamma^2h^2d,\ D_3=240\gamma h^4d^4+30h^3d/\mu\,.
\]}
\end{lemma}
The proof is found in~\citep{ding2020variance}. As in the case of an overdamped algorithm, one needs to tune $a$ and $c_p$ in the lemma above to obtain Theorem \ref{thm:thmconvergenceULDSAGA}. The result here is also similar to the previous ones. $T^m_2$ term enters in the Lyapunov function with an $h^2$ rate and is negligible and the iteration formula is dominated by the $D_3$ term in \eqref{TM+11BOUNDULD}. There are two main terms here, the $h^4d^4$ mainly comes from \blnote{$\EE|E^m|^2$} while the $h^3d$ is from the original iteration formula \eqref{eqn:imulmc}. It is hard to compare which term dominates. Indeed with smaller $\epsilon$, the requirement from the $h^3d$ term is more restrictive. This gives the last two terms in \eqref{eqn:convergeULDSAGA}.

\acks{Q.L. acknowledges support from Vilas Early Career award. The research is supported in part by NSF via grant DMS-1750488, TRIPODS 1740707 and 2023239 and Office of the Vice Chancellor for Research and Graduate Education at the University of Wisconsin Madison with funding from the Wisconsin Alumni Research Foundation.  Both authors appreciate valuable suggestions from the two anonymous reviewers and discussions with Stephen J. Wright.}

\newpage

\appendix
\section{Proof of Lemma~\ref{lem:E}}\label{CalculateofE}
We prove \eqref{momentboundforA} in this section. 
\blnote{
Recall the definition of $E^m$ to be
\[
E^m = \nabla f(x^m) - F^m\,.
\]
\\
To bound variance of $E^m$, we consider each component
\[
\EE(|E^m|^2)=\EE\left(\EE_r(|\partial_if(x^m)-d\partial_{r}f(x^m)\textbf{e}^{r}_i|^2)\right)=(d-1)\EE|\partial_if(x^m)|^2\,,
\]
where we use $r$ is uniformly chosen from $1,2,\dots,d$, this implies
\begin{equation}\label{momentboundforA1}
\EE(|E^m_i|^2)=\sum^d_{i=1}\EE(|E^m_i|^2)< \EE|\nabla f(x^m)|^2d\,.
\end{equation}}
Since gradient of $f$ is $L$-Lipschitz, we have
\begin{equation}\label{5.21}
\begin{aligned}
|\nabla f(x^m)|^2&\leq 2|\nabla f(x^m)-\nabla f(y^m)|^2+2|\nabla f(y^m)|^2\\
&\leq 2L^2|x^m-y^m|^2+2|\nabla f(y^m)|^2\,,
\end{aligned}
\end{equation}
and 
\begin{equation}\label{5.22}
\EE|\nabla f(y^m)|^2\leq dL\,,
\end{equation}
where we use $y_t\sim p$ for any $t$ and Lemma 3 in \citep{DALALYAN20195278} in the last inequality.
\blnote{
Plug \eqref{5.21},\eqref{5.22} into \eqref{momentboundforA1}, we have
\begin{equation}\label{momentboundforA2}
\EE(|E^m|^2)<2dL^2\EE|\Delta^m|^2+2d^2L\,,
\end{equation}
which proves \eqref{momentboundforA}}.

\section{Proof of Lemma \ref{lem:SVRGOLD}}\label{sec:proofoflem:SVRGOLD}

Let
\[
\wx=x^{s\tau},\quad \widetilde{y}=y^{s\tau}
\]
For $m=s\tau$, according to \eqref{randomgradientapproximationSVRG},\blnote{
\[
\EE\left|E^m\right|^2=(d-1)\EE\left|\nabla f(x^m)-\nabla f(\wx)\right|^2=0\,.
\]
Therefore, we only need to consider $m>s\tau$, according to \eqref{randomgradientapproximationSVRG}, we directly have
\begin{equation}\label{varofEXSVRG1}
\EE\left|E^m\right|^2=(d-1)\EE\left|\nabla f(x^m)-\nabla f(\wx)\right|^2\leq dL^2\EE\left|x^m-\wx\right|^2\,,
\end{equation}
where the last inequality comes from H\"older's inequality.} Then, use H\"older's inequality again, we have
\begin{equation}\label{varofEXSVRG2}
\EE\left|x^m-\wx\right|^2\leq 3\EE\left|x^m-y^m\right|^2+3\EE\left|y^m-\widetilde{y}\right|^2+3\EE\left|\widetilde{y}-\wx\right|^2\,,
\end{equation}
The second term can be further bounded by \eqref{eqn:ymolmc}
\begin{equation}\label{varofywy}
\EE\left|y^m-\widetilde{y}\right|^2=\EE\left|\int^{mh}_{s\tau h} \nabla f(y_s)\rd s-\sqrt{2h}\sum^{m-1}_{j=s\tau}\xi^{j}\right|^2\leq 2h^2\tau^2\EE_p|\nabla f(y)|^2+4h\tau d\,,
\end{equation}
where we use H\"older's inequality and $y_t\sim p$ for any $t$ in the last inequality. Finally, according to Lemma 3 in \citep{DALALYAN20195278} to obtain
\[
\EE_p|\nabla f(y)|^2\leq Ld
\]
Plug this into \eqref{varofywy}, combine \eqref{varofEXSVRG1},\eqref{varofEXSVRG2}, we obtain \eqref{varianceestimationofExmSVRGOPTION2}.

\section{Proof of Lemma \ref{lem:SVRGULD}}\label{sec:proofoflem:SVRGULD}
The proof is similar to Appendix \ref{sec:proofoflem:SVRGOLD}

For $m=s\tau$, according to \eqref{randomgradientapproximationSVRG},~\blnote{
\[
\EE\left|E^m\right|^2=(d-1)\EE\left|\nabla f(x^m)-\nabla f(x^{s\tau})\right|^2=0\,.
\]}
Therefore, we only need to consider $m>s\tau$, according to \eqref{randomgradientapproximationSVRG}, we directly have~\blnote{
\begin{equation}\label{varofEXSVRGULD1}
\EE\left|E^m\right|^2=(d-1)\EE\left|\nabla f(x^m)-\nabla f(x^{s\tau})\right|^2\leq dL^2\EE\left|x^m-x^{s\tau}\right|^2\,,
\end{equation}}
where the last inequality comes from H\"older's inequality. Then, use H\"older's inequality again, we have
\begin{equation}\label{varofEXSVRGULD2}
\EE\left|x^m-x^{s\tau}\right|^2\leq 3\EE\left|x^m-\wx^m\right|^2+3\EE\left|\wx^m-\wx^{s\tau}\right|^2+3\EE\left|\wx^{s\tau}-x^{s\tau}\right|^2\,,
\end{equation}
The second term can be further bounded by \eqref{eqn:ULDSDE2SAGAstar}
\begin{equation}\label{varofywyULD}
\EE\left|\wx^m-\wx^{s\tau}\right|^2=\EE\left|\int^{mh}_{s\tau h} \wrv_s\rd s\right|^2\leq 2h^2\tau^2\EE_{p_2}|\wrv|^2\,,
\end{equation}
where we use H\"older's inequality and $\wrv_t\sim p$ for any $t$ in the last inequality. Finally, since we have
\[
\EE_p|\wrv|^2=\gamma d\,,
\]
plug this into \eqref{varofywy}, combine \eqref{varofEXSVRG1},\eqref{varofEXSVRG2}, we obtain \eqref{varianceestimationofExmSVRGOPTION2ULD}.

\section{Key lemmas in the proof of Section \ref{proofofULMC}}\label{sec:keylammaulmc}
In this section, we always assume $h<1$ and $f$ satisfies assumption \ref{assum:Cov}, and all notations come from Section \ref{proofofULMC}.
\begin{proposition}\label{prop:boundJ1m} $\mathrm{J}^{r,m}_1,\mathrm{J}^m_2$ are defined in \eqref{eqn:J1rm} and \eqref{eqn:J2rm}, then for any $a>0$, we have
\blnote{
\begin{equation}\label{JRM1JM2}
\EE\left|\mathrm{J}^{r,m}_1\right|^2+\left|\mathrm{J}^m_2\right|^2\leq C_1\EE|\Delta^m|^2+5(1+1/a)\gamma^2h^4\EE(|E^m|^2)+5(1+1/a)\gamma h^4d\,,
\end{equation}}
where
\[
C_1=(1+a)[1-h/\kappa+200h^2]+80(1+1/a)h^4\,.
\]
\end{proposition}

We first introduce four new terms to prove Proposition \ref{prop:boundJ1m}. Denote
\begin{equation}\label{eqn:ASAGA}
\begin{aligned}
A^m=&(\wv^m-v^m)(h+e^{-2h})+(\wx^m-x^m)\\
&-\gamma\int^{(m+1)h}_{mh}e^{-2((m+1)h-s)}\left[\nabla f(\wx^m)-\nabla f(x^m)\right]\rd s\,,
\end{aligned}
\end{equation}
\begin{equation}\label{eqn:BSAGA}
\begin{aligned}
B^m=&\int^{(m+1)h}_{mh}\wrv_s-\rv_s-(\wv^m-v^m)\rd s\\
&-\gamma\int^{(m+1)h}_{mh}e^{-2((m+1)h-s)}\left[\nabla f\left(\wrx_s\right)-\nabla f(\wx^m)\right]\rd s
\end{aligned}\,,
\end{equation}
\begin{equation}\label{eqn:CSAGA}
C^m=(\wx^m-x^m)+\int^{(m+1)h}_{mh}\wv^m-v^m\rd s=(\wx^m-x^m)+h(\wv^m-v^m)\,,
\end{equation}
\begin{equation}\label{eqn:DSAGA}
D^m=\int^{(m+1)h}_{mh}\wrv_s-\rv_s-(\wv^m-v^m)\rd s\,,
\end{equation}
Then, we have the following three lemmas:

\begin{lemma}\label{lem:D1SAGA} For any $m\geq 0$
\begin{equation}\label{xboundSAGA}
\EE\int^{(m+1)h}_{mh}\left|\wrx_t-\wx^m\right|^2\rd t\leq \frac{h^3\gamma d}{3}
\end{equation}
and
\begin{equation}\label{vboundSAGA}
\begin{aligned}
\EE\int^{(m+1)h}_{mh}\left|\left(\wrv_t-\rv_t\right)-\left(\wv^m-v^m\right)\right|^2\rd t\leq &16h^3\EE|\Delta^m|^2+\gamma^2h^3\EE(|E^m|^2)+0.4\gamma h^5d\,,
\end{aligned}
\end{equation}
\end{lemma}
\begin{lemma}\label{lem:BDSAGA} $B^m,D^m$ are defined in \eqref{eqn:BSAGA},\eqref{eqn:DSAGA}, we have for any $m\geq 0$
\begin{equation}\label{bound:BSAGA}
\EE|B^m|^2\leq 32h^4\EE|\Delta^m|^2+2\gamma^2 h^4\EE(|E^m|^2)+2\gamma h^4 d
\end{equation}
\begin{equation}\label{bound:DSAGA}
\EE|D^m|^2\leq 16h^4\EE|\Delta^m|^2+\gamma^2h^4\EE(|E^m|^2)+0.4\gamma h^6d
\end{equation}
\end{lemma}
\begin{lemma}\label{lem:ACSAGA}
$A^m,C^m$ defined in \eqref{eqn:ASAGA},\eqref{eqn:CSAGA}, there exists a uniform constant $D$ such that for any $m\geq 0$
\blnote{
\begin{equation}\label{ACboundSAGA}
\EE(|A^m|^2+|C^m|^2)\leq \left[1-h/\kappa +200h^2\right]\EE|\Delta^m|^2\,
\end{equation}}
where $\kappa =L/\mu $ is the condition number of $f$.
\end{lemma}

Now, we are ready to prove Proposition \ref{prop:boundJ1m}:
\begin{proof}[Proof of Proposition \ref{prop:boundJ1m}]
First, we notice that
\[
\mathrm{J}^{r,m}_1=A^m+B^m,\quad \mathrm{J}^m_2=C^m+D^m\,.
\]
By Young's inequality, we have
\begin{equation}\label{eqn:ABCDSAGA}
\begin{aligned}
\EE|\mathrm{J}^{r,m}_1|^2+\EE|\mathrm{J}^m_2|^2=&\EE|A^m+B^m|^2+\EE|C^m+D^m|^2\\
\leq &(1+a)\left(\EE|A^m|^2+\EE|C^m|^2\right)\\
&+(1+1/a)(\EE|B^m|^2+\EE|D^m|^2)\,,
\end{aligned}
\end{equation}
where $a>0$ will be carefully chosen later. Now, the first term of \eqref{eqn:ABCDSAGA} only contains information from previous step, using $f$ is strongly convex, we can bound it using $\left|\Delta^m\right|^2$ (shown in Appendix \ref{sec:keylammaulmc} Lemma \ref{lem:ACSAGA}). To bound the second term, we need to consider difference between $x,v$ at $t_{m+1}$ and $t_m$, which can be bounded by $|\Delta^m|^2$ and $|E^m|^2$ (shown in Appendix \ref{sec:keylammaulmc} Lemma \ref{lem:BDSAGA}).

According to Lemma \ref{lem:BDSAGA}-\ref{lem:ACSAGA}, we first have
\[
\begin{aligned}
&\EE|\mathrm{J}^{r,m}_1|^2+\EE|\mathrm{J}^m_2|^2\\
\leq &(1+a)\left[1-h/\kappa+200h^2\right]\EE|\Delta^m|^2\\
&+(1+1/a)\left[80h^4\EE|\Delta^m|^2+5\gamma^2h^4\EE(|E^m|^2)+5\gamma h^4d\right]\\
= &C_1\EE|\Delta^m|^2+5(1+1/a)\gamma^2h^4\EE(|E^m|^2)+5(1+1/a)\gamma h^4d\,,
\end{aligned}
\]
where in the first inequality we use $1+h^2<2$ and
\[
C_1=(1+a)[1-h/\kappa+200h^2]+80(1+1/a)h^4\,.
\]
\end{proof}

To complete the proof, we prove three Lemmas one by one:
\begin{proof}[Proof of Lemma \ref{lem:D1SAGA}]
First we prove \eqref{xboundSAGA}. According to \eqref{eqn:ULDSDE2SAGAstar}, we have
\begin{equation}
\begin{aligned}
\EE\int^{(m+1)h}_{mh}\left|\wrx_t-\wx^m\right|^2dt&=\EE\int^{(m+1)h}_{mh}\left|\int^{t}_{mh} \wrv_sds\right|^2dt\\
&\leq \int^{(m+1)h}_{mh}(t-mh)\int^t_{mh} \EE\left|\wrv_s\right|^2dsdt\\
&=\int |v|^2p_2(x,v)\rd x\rd v\int^{(m+1)h}_{mh}(t-mh)^2dt=\frac{h^3\gamma d}{3}\,,
\end{aligned}
\end{equation}
where in the first inequality we use H\"older's inequality, and for the second equality we use $p_2$ is a stationary distribution so that $\left(\wrx_t,\wrv_t\right)\sim p_2$ and $\wrv_t\sim \exp(-|v|^2/(2\gamma))$ for any $t$.

Second, to prove \eqref{vboundSAGA}, using \eqref{eqn:ULDSDE2SAGA},\eqref{eqn:ULDSDE2SAGAstar}, we first rewrite $\left(\wrv_t-\rv_t\right)-\left(\wv^m-v^m\right)$ as
\begin{equation}\label{SSAGA}
\begin{aligned}
\left(\wrv_t-\rv_t\right)-\left(\wv^m-v^m\right)=&\left(\wv^m-v^m\right)(e^{-2(t-mh)}-1)\\
&-\gamma\int^t_{mh}e^{-2(t-s)}\left[\nabla f(\wrx_{s})-\nabla f(x^m)\right]\rd s\\
&+\gamma\int^t_{mh}e^{-2(t-s)}\rd s E^m\\
=&\mathrm{I}(t)+\mathrm{II}(t)+\mathrm{III}(t)\,.
\end{aligned}
\end{equation}
for $mh\leq t\leq (m+1)h$. Then we bound each term separately:
\begin{itemize}
\item 
\begin{equation}\label{S1SAGA}
\begin{aligned}
\EE\int^{(m+1)h}_{mh}\left|\mathrm{I}(t)\right|^2\rd t&\leq h\EE\int^{(m+1)h}_{mh} \left|\left(\wv^m-v^m\right)(e^{-2(t-mh)}-1)\right|^2\rd t\\
&\leq h\int^{(m+1)h}_{mh} (2(t-mh))^2\EE\left|\wv^m-v^m\right|^2\rd t\\
&\leq \frac{4h^3}{3}\EE\left|\wv^m-v^m\right|^2\,,
\end{aligned}
\end{equation}
where we use H\"older's inequality in the first inequality and $1-e^{-x}<x$ in the second inequality.

\item
\begin{equation}\label{S2SAGA}
\begin{aligned}
&\EE\int^{(m+1)h}_{mh}\left|\mathrm{II}(t)\right|^2\rd t\leq \gamma^2\EE\int^{(m+1)h}_{mh}\left|\int^t_{mh}e^{-2(t-s)}\left[\nabla f(\wrx_s)-\nabla f(x^m)\right]\rd s\right|^2\rd t\\
\leq &2\gamma^2\EE\int^{(m+1)h}_{mh}\left|\int^t_{mh}e^{-2(t-s)}\left[\nabla f(\wrx_s)-\nabla f(\wx^m)\right]\rd s\right|^2\rd t\\
&+2\gamma^2\EE\int^{(m+1)h}_{mh}\left|\int^t_{mh}e^{-2(t-s)}\left[\nabla f(\wx^m)-\nabla f(x^m)\right]\rd s\right|^2\rd t\\
\leq &2\gamma^2\int^{(m+1)h}_{mh}(t-mh)\EE\int^t_{mh}\left|\nabla f(\wrx_s)-\nabla f(\wx^m)\right|^2\rd s \rd t\\
&+2\gamma^2\int^{(m+1)h}_{mh}(t-mh)\EE\int^t_{mh}\left|\nabla f(\wx^m)-\nabla f(x^m)\right|^2\rd s \rd t\\
\leq &2\gamma^2 L^2\int^{(m+1)h}_{mh}(t-mh)\EE\int^t_{mh}\left|\wrx_s-\wx^m\right|^2\rd s \rd t\\
&+2\gamma^2 L^2\int^{(m+1)h}_{mh}(t-mh)\EE\int^t_{mh}\left|\wx^m-x^m\right|^2\rd s \rd t\\
\leq &2\gamma^3 L^2d\int^{(m+1)h}_{mh}\frac{(t-mh)^4}{3}\rd t+2\gamma^2 L^2\int^{(m+1)h}_{mh}(t-mh)^2\rd t\EE\left|\wx^m-x^m\right|^2\\
\leq &\frac{2\gamma^3 L^2h^5d}{15}+\frac{2\gamma^2 L^2h^3}{3}\EE\left|\wx^m-x^m\right|^2\,,
\end{aligned}
\end{equation}
where in the third inequality we use the gradient of $f$ is $L$-Lipschitz function and we use \eqref{xboundSAGA} in the fourth inequality.

\item 
\begin{equation}\label{S3SAGA}
\begin{aligned}
\EE\int^{(m+1)h}_{mh}\left|\mathrm{III}(t)\right|^2\rd t&=\gamma^2\EE\int^{(m+1)h}_{mh}\left|\int^t_{mh}e^{-2(t-s)}\rd s E^m\right|^2\rd t\\
&\leq \gamma^2\int^{(m+1)h}_{mh}(t-mh)^2\rd t\EE(|E^m|^2)\\
&\leq \frac{\gamma^2h^3}{3}\EE(|E^m|^2)\,,
\end{aligned}
\end{equation}
\end{itemize}
Plug \eqref{S1SAGA},\eqref{S2SAGA},\eqref{S3SAGA} into \eqref{SSAGA} and using $\gamma L=1$, we have
\[
\begin{aligned}
&\EE\int^{(m+1)h}_{mh}\left|\left(\wrv_t-\rv_t\right)-\left(\wv^m-v^m\right)\right|^2\rd t\\
\leq &3\left(\EE\int^{(m+1)h}_{mh}\left|\mathrm{I}(t)\right|^2\rd t+\EE\int^{(m+1)h}_{mh}\left|\mathrm{II}(t)\right|^2\rd t+\EE\int^{(m+1)h}_{mh}\left|\mathrm{III}(t)\right|^2\rd t\right)\\
\leq & 4h^3\left(\EE\left|\wx^m-x^m\right|^2+\EE\left|\wv^m-v^m\right|^2\right)+\gamma^2h^3\EE(|E^m|^2)+0.4\gamma h^5d\,,
\end{aligned}
\]
using \eqref{trivialinequlaitySAGA}, we get the desired result.
\end{proof}

\begin{proof}[Proof of Lemma \ref{lem:BDSAGA}]
First, we separate $B^m$ into two parts:
\[
\begin{aligned}
\EE|B^m|^2\leq &2\EE\left|\int^{(m+1)h}_{mh} \left(\wrv_t-\rv_t\right)-\left(\wv^m-v^m\right)\rd t\right|^2\\
&+2\EE\left|\gamma\int^{(m+1)h}_{mh}e^{-2((m+1)h-t)}\left[\nabla f(\wrx_t)-\nabla f(\wx^m)\right]\rd t\right|^2\,.\\
\end{aligned}
\]
And each terms can be bounded:
\begin{itemize}
\item
\begin{equation}\label{V1boundSAGA}
\begin{aligned} 
&\EE\left|\int^{(m+1)h}_{mh} \left(\wrv_t-\rv_t\right)-\left(\wv^m-v^m\right)\rd t\right|^2\\
\leq &h\EE\int^{(m+1)h}_{mh}\left|\left(\wrv_t-\rv_t\right)-\left(\wv^m-v^m\right)\right|^2\rd t\\
\leq &16h^4\EE|\Delta^m|^2+\gamma^2h^4\EE(|E^m|^2)+0.4\gamma h^6d\,,
\end{aligned}
\end{equation}
where we use Lemma \ref{lem:D1SAGA} \eqref{vboundSAGA} in the second inequality.
\item
\begin{equation}\label{V2boundSAGA}
\begin{aligned} 
&\EE\left|\gamma\int^{(m+1)h}_{mh}e^{-2((m+1)h-t)}\left[\nabla f(\wrx_t)-\nabla f(\wx^m)\right]\rd t\right|^2\\
\leq &h\gamma^2\EE\int^{(m+1)h}_{mh}\left|e^{-2((m+1)h-t)}\left[\nabla f(\wrx_t)-\nabla f(\wx^m)\right]\right|^2\rd t\\
\leq &h\gamma^2L^2\EE\int^{(m+1)h}_{mh}\left|\wrx_t-\wx^m\right|^2\rd t\\
\leq &\frac{h^4\gamma^3L^2 d}{3}\leq \frac{h^4\gamma d}{3}\,,
\end{aligned}
\end{equation}
where we use Lemma \ref{lem:D1SAGA} \eqref{xboundSAGA} and $\gamma L=1$ in the last two inequalities. 

\end{itemize}
Combine \eqref{V1boundSAGA},\eqref{V2boundSAGA} together, we finally have
\[
\EE|B|^2\leq 32h^4\EE|\Delta^m|^2+2\gamma^2 h^4\EE(|E^m|^2)+0.8 h^6\gamma d+2h^4\gamma d/3\,,
\]
which implies \eqref{bound:BSAGA} if we further use $h<1$.

Next, estimation of $\left(\EE|D|^2\right)^{1/2}$ is a direct result of \eqref{V1boundSAGA}.
\end{proof}

\begin{proof}[Proof of Lemma \ref{lem:ACSAGA}]
Let $\wx^m-x^m=a$ and $\ww^m-w^m=b$. First, by the mean-value theorem, there exists a matrix $H$ such that $\mu {I}_d\preceq H\preceq L{I}_d$ and
\[
\nabla f(\wx^m)-\nabla f(x^m)=Ha\,.
\]
By calculation, $\int^{(m+1)h}_{mh}e^{-2((m+1)h-t)}\rd t=\frac{1-e^{-2h}}{2}$ and 
\[
\begin{aligned}
A^m&=(h+e^{-2h})(\wv^m-v^m)+\left(I_{d}-\frac{(1-e^{-2h})}{2}\gamma H\right)(\wx^m-x^m)\\
&=\left(\left(1-h-e^{-2h}\right)I_{d}-\frac{(1-e^{-2h})}{2}\gamma H\right)a+(h+e^{-2h})b
\end{aligned}\,.
\]
\[
C^m=(1-h)a+hb\,.
\]
Since $\|\gamma H\|_2\leq 1$ and we also have the following calculation
\[
|1-e^{-2h}-2h|\leq 5h^2\,,
\]
where we use $h<\frac{1}{1648}.$

If we further define matrix $\mathcal{M}_A$ and $\mathcal{M}_C$ such that
\[
|A^m|^2=\left(a,b\right)^\top \mathcal{M}_A\left(a,b\right)\,,\quad |C^m|^2=\left(a,b\right)^\top \mathcal{M}_C\left(a,b\right)\,,
\]
then, we have
\[
\left\|\mathcal{M}_A-\begin{bmatrix}
0 & hI_{d}-\gamma hH\\
hI_{d}-\gamma hH & (1-2h)I_{d}
\end{bmatrix}\right\|_2\leq 100h^2\,,
\]
and
\[
 \left\|\mathcal{M}_B-\begin{bmatrix}
(1-2h)I_{d} & hI_{d}\\
hI_{d} & 0
\end{bmatrix}\right\|_2\leq 100h^2\,,
\]
where we use $h<1/1648$ by \eqref{eqn:conditionuetaULDSAGA} and $\|\gamma H\|_2\leq 1$. This further implies\blnote{
\[
|A^m|^2+|C^m|^2=\left(a,b\right)^\top\begin{bmatrix}
(1-2h)I_{d} & 2hI_{d}-\gamma hH\\
2hI_{d}-\gamma hH & (1-2h)I_{d}
\end{bmatrix}\left(a,b\right)+h^2\left(a,b\right)^\top Q\left(a,b\right)
\]}
where $\|Q\|_2\leq 200$. Calculate the eigenvalue of the dominating matrix (first term), we need to solve
\[
\mathrm{det}\left\{(1-2h-\lambda)^2I_{d}-(2hI_{d}-\gamma hH)^2\right\}=0\,,
\]
which implies eigenvalues $\{\lambda_j\}^d_{j=1}$ solve
\[
(1-2h-\lambda_j)^2-(2h-\gamma h\Lambda_j)^2=0\,,
\]
where $\Lambda_j$ is $j$-th eigenvalue of $H$. Since $\gamma\Lambda_j\leq \gamma L=1$ and $h<1$, we have 
\[
\lambda_j\leq 1-\gamma\Lambda_jh\leq 1-\mu h\gamma=1-h/\kappa 
\]
for each $j=1,\dots,d$. This implies
\[
\left\|\begin{bmatrix}
(1-2h)I_{d} & 2hI_{d}-\gamma hH\\
2hI_{d}-\gamma hH & (1-2h)I_{d}
\end{bmatrix}\right\|_2\leq 1-h/\kappa \,,
\]
and
\blnote{
\[
|A^m|^2+|C^m|^2\leq (1-h/\kappa +200h^2)(|a|^2+|b|^2)\,.
\]}
Taking expectation on both sides, we obtain \eqref{ACboundSAGA}.

\end{proof}

\vskip 0.2in
\bibliography{enkf}

\end{document}